\documentclass{article}

    \usepackage[final]{neurips_2024}

\usepackage[utf8]{inputenc} %
\usepackage[T1]{fontenc}    %
\usepackage{hyperref}       %
\usepackage{url}            %
\usepackage{booktabs}       %
\usepackage{amsfonts}       %
\usepackage{nicefrac}       %
\usepackage{microtype}      %
\usepackage{xcolor}         %
\usepackage{graphicx}         

\usepackage{amsmath,xcolor,mathrsfs}
\usepackage{bbm}
\usepackage{bm}
\usepackage{xspace}
\usepackage{wrapfig}

\usepackage{subcaption}
\newcommand\SmallMatrix[1]{{%
    \tiny\arraycolsep=0.3\arraycolsep\ensuremath{\begin{pmatrix}#1\end{pmatrix}}}}

\title{Interaction-Force Transport Gradient Flows}

\author{%
Egor Gladin\\
    Humboldt University of Berlin \\
    Berlin, Germany \\
    \& HSE University\\
  \texttt{egorgladin@yandex.ru} \\
  \And
  Pavel Dvurechensky\\
   Weierstrass Institute for\\
   Applied Analysis and Stochastics \\
   Berlin, Germany \\
  \texttt{pavel.dvurechensky@wias-berlin.de} \\
  \AND
  Alexander Mielke \\
  Humboldt University of Berlin\\
  \& WIAS\\
  Berlin, Germany\\
  \texttt{alexander.mielke@wias-berlin.de} \\
  \And
  Jia-Jie Zhu\thanks{Corresponding author: Jia-Jie Zhu}\\
   Weierstrass Institute for\\
   Applied Analysis and Stochastics \\
   Berlin, Germany \\
  \texttt{jia-jie.zhu@wias-berlin.de} \\
}

\renewcommand{\cite}{\citep}

\usepackage{enumitem} %
\usepackage{amsthm}
\newtheorem{theorem}{Theorem}[section]
\newtheorem{proposition}[theorem]{Proposition}

\newtheorem{corollary}[theorem]{Corollary}

\newtheorem{remark}[theorem]{Remark}

\usepackage{preamble-JZ}

\newcommand{\Mplus}{\mathcal{M}^+}
\newcommand{\ours}{\ensuremath{\mathrm{IFT}}\xspace}

\newcommand{\rkhs}{\mathcal{H}}
\newcommand{\XX}{\boldsymbol{X}} %

\begin{document}

\newpage
\maketitle

\begin{abstract}
This paper presents a new gradient flow dissipation geometry over non-negative and probability measures. This is motivated by a principled construction that combines the unbalanced optimal transport and interaction forces modeled by reproducing kernels. Using a precise connection between the Hellinger geometry and the maximum mean discrepancy (MMD), we propose the interaction-force transport (IFT) gradient flows and its spherical variant via an infimal convolution of the Wasserstein and spherical MMD tensors. We then develop a particle-based optimization algorithm based on the JKO-splitting scheme of the mass-preserving spherical IFT gradient flows. Finally, we provide both theoretical global exponential convergence guarantees and improved empirical simulation results for applying the IFT gradient flows to the sampling task of MMD-minimization. Furthermore, we prove that the spherical IFT  gradient flow enjoys the best of both worlds by providing the global exponential convergence guarantee for both the MMD and KL energy.
\end{abstract}

\section{Introduction}
Optimal transport (OT) distances between probability measures, including the earth mover's distance \cite{werman1985distance,rubner2000earth} and Monge-Kantorovich or Wasserstein distance \cite{villani2008optimal}, are one of the cornerstones of modern machine learning as they allow performing a variety of machine learning tasks, e.g., unsupervised learning \cite{arjovsky2017wasserstein,bigot2017geodesic}, semi-supervised learning \cite{solomon2014wasserstein}, clustering \cite{ho17multilevel}, text classification \cite{kusner2015from}, image retrieval, clustering and classification \cite{rubner2000earth,cuturi2013sinkhorn,sandler2011nonnegative},
and distributionally robust optimization \cite{sinhaCertifyingDistributionalRobustness2020,mohajerinesfahaniDatadrivenDistributionallyRobust2018}.
Many recent works in machine learning adopted the techniques from PDE gradient flows over optimal transport geometries
and interacting particle systems for inference and sampling tasks.
Those tools not only add new interpretations to the existing algorithms, but also provide a new perspective on designing new algorithms.

For example,
the classical Bayesian inference framework minimizes
the Kulback-Leibler divergence 
towards a target distribution $\pi$.
From the optimization perspective,
this can be viewed as solving
\begin{align}
\min_{\mu\in A\subset \cal P}
   \big\{ F(\mu ) := \mathrm{D}_\textrm{KL}(\mu| \pi)\big\},
   \label{eq:kl-inference}
\end{align}
\begin{wrapfigure}{r}{0.47\textwidth}
    \centering
    \includegraphics[width=\linewidth]{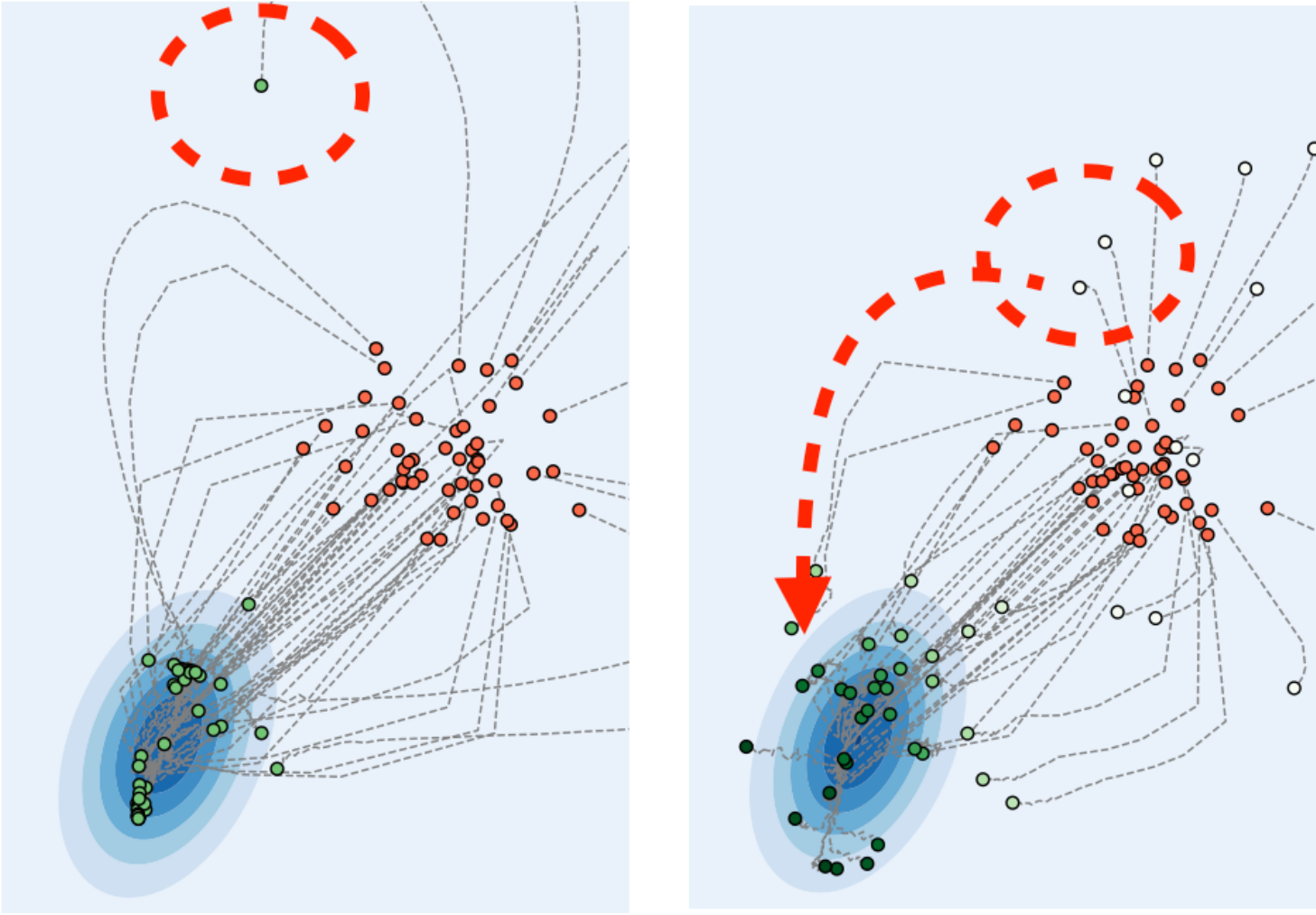}
    \caption{(Left) Wasserstein flow of the MMD energy ~\citep{arbel_maximum_2019}. Some particles get stuck at points away from the target.
    (Right) \ours gradient flow (this paper) of the MMD energy. Particle mass is teleported to close to the target, avoiding local minima. Hollow circles indicate particles with zero mass.
    The red dots are the initial particles, and the green dots are the target distribution.
    See \S\ref{sec:numerical-example} for more details.
    }
    \label{fig:teleport-intro}
\end{wrapfigure}
where $A$ is a subset of the space of probability measures $\cal P$, \eg the Gaussian family.
The Wasserstein gradient flow of the KL gives the Fokker-Planck equation, which can be simulated using the
Langevin SDE for MCMC.
Beyond the KL,
many researchers following
\citet{arbel_maximum_2019} advocated using the squared MMD instead as the driving energy
for the Wasserstein gradient flows for sampling.
However, in contrast to the KL setting,
there is little sound convergence analysis for the MMD-minimization scheme like the celebrated \BE Theorem.
Furthermore, it was shown, \eg in \citep{korbaKernelSteinDiscrepancy2021},
that \citet{arbel_maximum_2019}'s algorithm suffers a few practical drawbacks.
For example,
their
particles tend to collapse around the mode or get stuck at local minima,
and the algorithm requires a heuristic noise injection strategy that is tuned over the iterations;
see 
Figure~\ref{fig:teleport-intro} 
and
\S\ref{sec:numerical-example} for illustrations.
Subsequently, many such as \citet{carrillo2019blob,chewiSVGDKernelizedWasserstein2020,korbaKernelSteinDiscrepancy2021,glaserKALEFlowRelaxed2021,craig_blob_2023,hertrich2023generative,neumayer2024wasserstein}
proposed modified energies to be used in the Wasserstein gradient flows.
In contrast, this paper does not propose new energy objectives.
Instead,
we propose a new gradient flow geometry -- the \ours gradient flows.
To summarize, our main {contributions} are:
\begin{enumerate}[noitemsep, topsep=0pt, leftmargin=*]
    \item We propose the
    interaction-force transport (\ours)
    gradient flow geometry over non-negative measures and spherical 
    \ours over probability measures, constructed from the first principles of the reaction-diffusion type equations, previously studied in the context of the Hellinger-Kantorovich (Wasserstein-Fisher-Rao) distance and gradient flows.
    It was first studied by three groups including~\citet{chizatInterpolatingDistanceOptimal2018,chizat_unbalanced_2019,liero_optimal_2018,kondratyevNewOptimalTransport2016,gallouet2017jko}.
    Our \ours gradient flow is based on the inf-convolution of the Wasserstein and 
    the newly constructed spherical MMD Riemannian metric tensors.
    This new unbalanced gradient flow geometry allows teleporting particle mass in addition to transportation,
    which avoids the flow getting stuck at local minima; see Figure~\ref{fig:teleport-intro} for an illustration.
    \item We provide 
    theoretical analysis such as the 
    \emph{global exponential decay} of energy functionals
    via the Polyak-\Loj type functional inequalities.
    As an application,
    we provide the first global exponential convergence analysis of \ours for both
    the MMD and KL energy functionals.
    That is, the \ours gradient flow enjoys the best of both worlds.
    \item We provide a new algorithm for the implementation of the \ours gradient flow.
    We then empirically demonstrate the use of the \ours gradient flow for the MMD inference task.
    Compared to the original MMD-energy-flow algorithm
    of \citet{arbel_maximum_2019}, \ours flow does not suffer issues
    such as the collapsing-to-mode issue.
    Leveraging the first-principled spherical \ours gradient flow, 
    our method does not require a heuristic noise injection that is commonly tuned over the iterations in practice; see ~\citep{korbaKernelSteinDiscrepancy2021} for a discussion.
    Our method can also be viewed as addressing a long-standing issue of the kernel-mean embedding methods~\citep{smolaHilbertSpaceEmbedding2007,muandetKernelMeanEmbedding2017,lacoste-julienSequentialKernelHerding} for optimizing the support of distributions.
\end{enumerate}
\textbf{Notation}
We use the notation $\mathcal{P}(\XX), \Mplus(\XX)$ to denote the space of probability and non-negative measures on the closed, bounded, (convex) set $\XX\subset\mathbb R^d$.
The base space symbol $\XX$ is often dropped if there is no ambiguity in the context. We note also that many of our results hold for $\XX=\mathbb R^d$.
In this paper, the first variation of a functional $F$ at $\mu\in\Mplus$ is defined as a function ${\frac{\delta F}{\delta\mu}[\mu] }$
such that
$\frac{\dd }{\dd \epsilon}F(\mu + \epsilon \cdot v) |_{\epsilon=0}
= \int {{\frac{\delta F}{\delta\mu}[\mu] }}(x) \dd v (x)$
for any valid perturbation in measure $v$ such that $\mu + \epsilon \cdot v\in \Mplus$ when working with gradient flows over $\Mplus$ and 
$\mu + \epsilon \cdot v\in \mathcal P$ over $\mathcal P$.
We often omit the time index $t$ to lessen the notational burden, \eg the measure at time $t$, $\mu(t, \cdot)$, is written as $\mu$.
The infimal convolution (inf-convolution) of two functions $f,g$ on Banach spaces is defined as
    $(f\square g)(x) = \inf_{y} \left\{f(y) + g(x-y)\right\}$.
In formal calculation,
we often use measures and their density interchangeably,
\ie$\int f\cdot \mu$ means the integral w.r.t. the measure $\mu$.
For a rigorous generalization of flows over continuous measures to discrete measures, see \citep{ambrosio2008gradient}.
$\nabla_2 k(\cdot, \cdot) $ denotes the gradient w.r.t. the second argument of the kernel.
\section{Background}

\subsection{Gradient flows of probability measures for learning and inference}
Gradient flows are powerful tools originated from the field of PDE.
The intuition can be easily seen 
from the perspective of optimization as solving the variational problem
\begin{align*}
    \min _{\mu\in A\subset \Mplus(\XX)} F(\mu )
\end{align*}
using a ``continuous-time version'' of gradient descent,
over a suitable metric space and, in particular, Riemannian manifold.
Since the seminal works by \citet{otto1996double} and colleagues,
one can view many PDEs as
gradient flows over the aforementioned Wasserstein metric space, canonically denoted as $\left(\mathcal P_2(\XX), W_2\right)$;
see \citep{villani2008optimal,santambrogio_optimal_2015}
for a comprehensive introduction.

Different from a standard OT problem, a gradient flow solution traverses along the path of the fastest dissipation of the energy $F$ allowed by the corresponding 
geometry.
In this paper, we are only concerned with the
geometries with a
(pseudo-)Riemannian structure, such as the Wasserstein, (spherical) Hellinger or Fisher-Rao geometries.
In such cases, a formal Otto calculus can be developed to greatly simplify the calculations.
For example, the Wasserstein Onsager operator (which is the inverse of the Riemannian metric tensor)
$\mathbb K_W(\rho): 
T^*_\rho \Mplus \to T_\rho \Mplus, \xi \mapsto -\DIV(\rho\nabla \xi)$,
where $T_\rho \Mplus$ is the tangent plane of $\Mplus$ at $\rho$ and
$T^*_\rho \Mplus$ the cotangent plane.
Using this notation, a Wasserstein gradient flow equation of some energy $F$ can be written as
\begin{align}
    \dot \mu = -  \mathbb K_W(\mu) \frac{\delta F}{\delta \mu}
    = \DIV(\mu\nabla \frac{\delta F}{\delta \mu})
    .
    \label{eq:wasserstein-gfe}
\end{align}
In essence, many machine learning applications
are about making different choices of the energy $F$ in \eqref{eq:wasserstein-gfe}, \eg the KL, $\chi^2$-divergence, or MMD.
However, Wasserstein and its flow equation~\eqref{eq:wasserstein-gfe} are by no means the only meaningful geometry for gradient flows.
One major development in the field is the Hellinger-Kantorovich a.k.a. the Wasserstein-Fisher-Rao (WFR) gradient flow.
    The WFR gradient flow equation is given by the reaction-diffusion equation,
    for some scaling coefficients
    $\alpha, \beta>0$,
    \begin{align}
        \dot u =  \alpha\cdot \DIV (u\nabla \frac{\delta F}{\delta u}  ) - \beta u \cdot\frac{\delta F}{\delta u}
        .
        \label{eq:wfr-gfe}
    \end{align}
A few recent works have applied WFR to sampling and inference~\citep{yanLearningGaussianMixtures2023,luAcceleratingLangevinSampling2019} by choosing the energy functional to be the KL divergence.

\subsection{Reproducing kernel Hilbert space and MMD}
In this paper, we refer to 
a bi-variate function $k: \XX \times \XX \rightarrow \mathbb{R}$ as a symmetric positive definite kernel if $k$ is symmetric and, for all $n \in \mathbb{N}, \alpha_1, \ldots, \alpha_n \in \mathbb{R}$ and all $x_1, \ldots, x_n \in \XX$, we have $\sum_{i=1}^n \sum_{j=1}^n \alpha_i \alpha_j k\left(x_j, x_i\right) \geq 0$.
$k$ is a reproducing kernel if it satisfies the reproducing property, i.e., for all $x \in \XX$ and all functions in a Hilbert space $f \in \rkhs$, we have $f(x)=\langle f, k(\cdot, x)\rangle_{\rkhs}$.
Furthermore, the space $\rkhs$ is an RKHS if the Dirac functional $\delta_x: \rkhs \mapsto \mathbb{R}, \delta_x(f):=f(x)$ is continuous.
It can be shown that there is a one-to-one correspondence between the RKHS $\rkhs$ and the reproducing kernel $k$.
Suppose the kernel is square-integrable $\|k\|_{L^2_{\rho}}^2:=\int k(x, x) d \rho(x)<\infty$ w.r.t. $\rho\in \mathcal P$.
The integral operator \(\Tkrho: L^2_{\rho} \rightarrow \rkhs\) is defined by
$
\Tkrho g(x):=\int k\left(x, x^{\prime}\right) g\left(x^{\prime}\right) d \rho\left(x^{\prime}\right)$ for $g \in L^2_{\rho}
$.
With an abuse of terminology, we refer to the following composition also as the integral operator
$$\K_\rho:= \ID \circ \Tkrho,\
L^2(\rho) \to L^2(\rho).$$
$\K_\rho$ is compact, positive, self-adjoint, and nuclear; cf. \citep{steinwart2008support}.
To simplify the notation, we simply write $\K$ when $\rho$ is the Lebesgue measure.

The kernel maximum mean discrepancy (MMD)~\cite{gretton2012kernel} emerged as
an easy-to-compute alternative 
to optimal transport
for computing the distance between probability measures,
\ie
    $\displaystyle\mmd^2( \mu, \nu): = \|\K \left(\mu - \nu\right)\|_{\rkhs}^2
    =\int \int k(x, x') \dd (\mu - \nu)(x) \dd (\mu - \nu)(x')$
    ,
where $\rkhs$ is the RKHS associated with the (positive-definite) kernel $k$.
While the MMD enjoys many favorable properties, such as a closed-form estimator
and favorable statistical properties~\cite{tolstikhinMinimaxEstimationKernel2017,tolstikhinMinimaxEstimationMaximum},
its mathematical theory is less developed compared to the Wasserstein space especially in the geodesic structure and gradient flow geometries.
It has been shown by \citet{zhu2024approximation}
that MMD is a (de-)kernelized Hellinger or Fisher-Rao distance
by using a dynamic formulation
\begin{equation}
    \label{eq:bb-formula-mmd}
    \begin{aligned}
        \mmd^2(\mu,\nu)
        =
        \min
\left\{        \int_0^1
        \|  \xi_t \|^2_{\rkhs}
        \dd t
        \
        \middle \vert \ 
         \dot u = - \K^{-1}   \xi_t,
          u(0) =  \mu,
         u(1)=  \nu,\
         \xi_t\in  \rkhs
         \right\}
         .
    \end{aligned}
\end{equation}
Mathematically, we can obtain the MMD geodesic structure if we kernelize the Hellinger (Fisher-Rao) Riemannian metric tensor,
\begin{align}
    \mathbb G_{\mmd} = \K_\mu\circ\mathbb G_\He (\mu), \quad \mathbb K_{\mmd} = \mathbb K_\He (\mu) \circ \K_\mu^{-1}
    ,
    \label{eq:metric-tensor-relation}
\end{align}
noting that the Onsager operator $\mathbb K$ is the inverse of the Riemannian metric tensor $\mathbb K = \mathbb G^{-1}$.
The MMD suffers from some shortcomings in practice, such as the vanishing gradients and kernel choices that require careful tuning; see \eg \cite{feydyInterpolatingOptimalTransport2019}.
Furthermore,
a theoretical downside of the MMD
as a tool for optimizing distributions, and kernel-mean embedding~\cite{smolaHilbertSpaceEmbedding2007,muandetKernelMeanEmbedding2017} in general, is that they do not allow \emph{transport} dynamics.
This limitation is manifested in practice, \eg it is intractable to optimize the location of particle distributions; see e.g. \citep{lacoste-julienSequentialKernelHerding}.
In this paper, we address all those issues.

\section{\ours gradient flows over non-negative and probability measures}
\begin{wrapfigure}{r}{0.55\textwidth}
      \centering
      \includegraphics[width=0.53\textwidth]{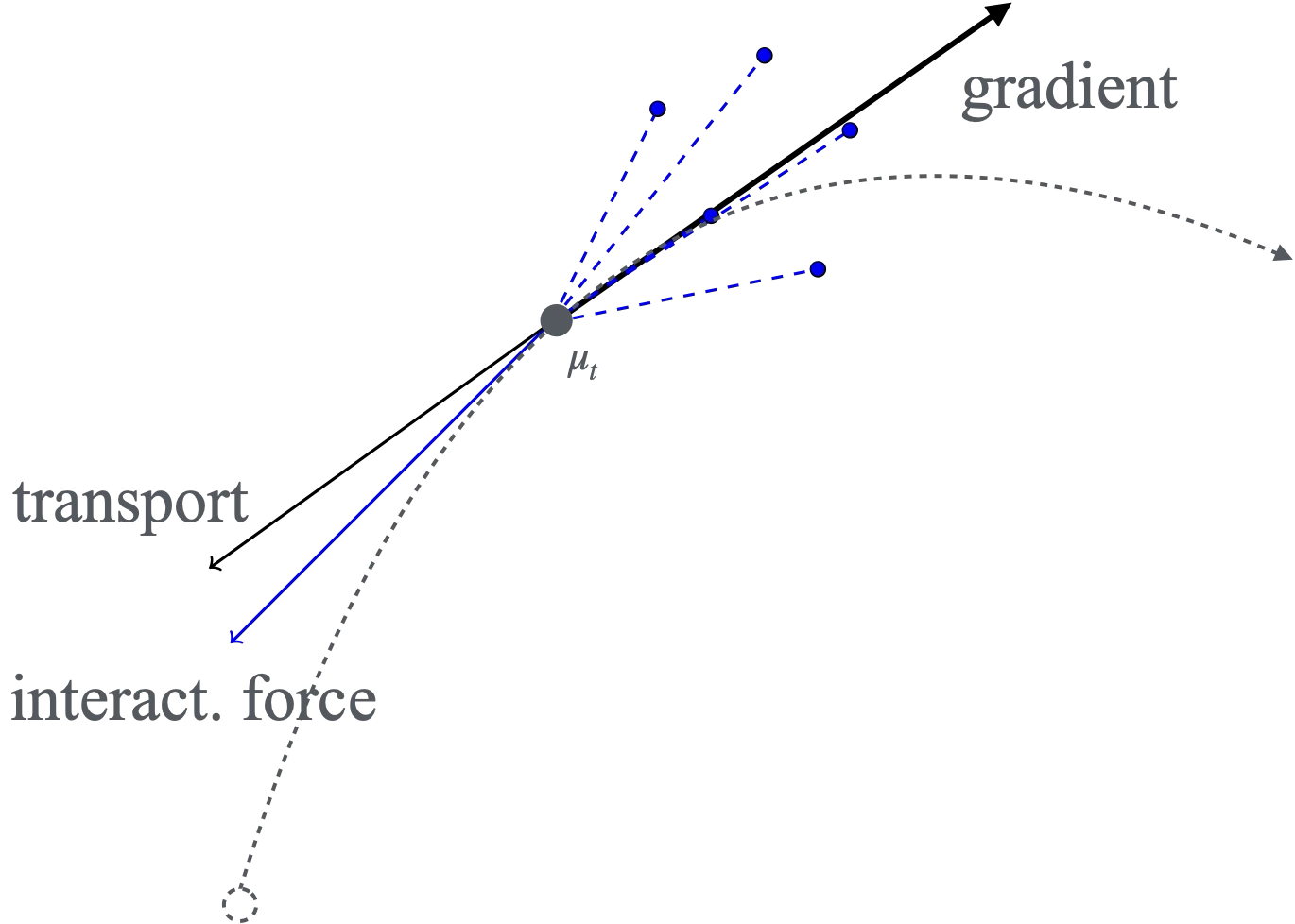}
      \caption[\ours intuition]{Illustration of the \ours gradient flow.    
      Atoms are subject to both the transport (Kantorovich) potential and the interaction (repulsive) force from other atoms.}
      \label{fig:int-force-gf}
\end{wrapfigure}
In this section, we
propose the \ours
gradient flows over 
non-negative and
probability measures.
Note that our 
methodology
is fundamentally
different from a few related works in kernel methods and gradient flows
such as
\cite{arbel_maximum_2019,korbaKernelSteinDiscrepancy2021,hertrich2023generative,glaserKALEFlowRelaxed2021,neumayer2024wasserstein}
in that we are not concerned with the Wasserstein flows of a different energy,
but a new gradient flow dissipation geometry.
\subsection{(Spherical) \ours gradient flow equations over non-negative and probability measures}

    The construction of the 
    Wasserstein-Fisher-Rao gradient flows crucially relies on the inf-convolution from convex analysis~\citep{liero_optimal_2018,chizat2022sparse}.
There, the WFR metric tensor is defined using an inf-convolution of the Wasserstein tensor and the Hellinger (Fisher-Rao) tensor
    $\mathbb G_{\WFR}(\mu) = \mathbb G_{W}(\mu) \square \mathbb G_{\mathrm{He}}(\mu)$.
By Legendre transform, its inverse, the Onsager operator, is given by the sum
    $\mathbb K_{\WFR}(\mu) = \mathbb K_{W}(\mu) + \mathbb K_{\mathrm{He}}(\mu)$.
Therefore, we construct the \ours gradient flow by replacing the Hellinger (Fisher-Rao) tensor with the MMD tensor, as in \eqref{eq:metric-tensor-relation}.
\begin{align}
    \mathbb G_{\ours}(\mu) = \mathbb G_{W}(\mu) \square \mathbb G_{\mmd}(\mu),\quad
    \mathbb K_{\ours}(\mu) = \mathbb K_{W}(\mu) + \mathbb K_{\mmd}(\mu).
    \label{eq:metric-tensor-relation-ift}
\end{align}
The MMD gradient flow equation is
derived by \citet{zhu2024approximation}
using the Onsager operator~\eqref{eq:metric-tensor-relation},
\begin{align}
    \dot \mu = -\mathbb K_{\mmd}(\mu)\dFdmu=-   \K^{-1} \dFdmu.
    \label{eq:mmd-gfe}
\end{align}
Hence, we obtained the desired \ours gradient flow equation using \eqref{eq:metric-tensor-relation-ift}.
\begin{align}
    \label{eq:ikw-gfe-unreg}
    \dot \mu =  
    -\alpha \mathbb K_{W}(\mu)\dFdmu
    -\beta \mathbb K_{\mmd}(\mu)\dFdmu
    =
    {\alpha}\cdot \operatorname{div}(\mu \nabla\dFdmu) -   {\beta} \cdot \K^{-1} \dFdmu.
\end{align}

Formally,
the \ours gradient flow equation
can also be viewed as a kernel-approximation to
the Wasserstein-Fisher-Rao gradient flow equation,
\ie the reaction-diffusion equation~\eqref{eq:wfr-gfe}.
\begin{corollary}
    Suppose $\int k_\sigma(x,\cdot )\dd \mu =1$ and
    the kernel-weighted-measure converges to the Dirac measure $k_\sigma(x, \cdot)\dd \mu \to\dd \delta_x $ as the bandwidth $\sigma\to 0$.
    Then, the
    \ours gradient flow equation~\eqref{eq:ikw-gfe-unreg}
    tends towards
    the WFR gradient flow equation as 
    $\sigma\to 0$,
    \ie the reaction-diffusion equation~\eqref{eq:wfr-gfe}.
    \label{cor:kernel-approx-ours}
\end{corollary}
Like the WFR gradient flow over non-negative measures,
the gradient flow equation \eqref{eq:ikw-gfe-unreg} and \eqref{eq:mmd-gfe} are not guaranteed to stay within the probability measure space, i.e., total mass $1$.
This is useful in many applications such as chemical reaction systems.
However, 
probability measures are often required for
machine learning applications.
We now provide a mass-preserving gradient flow equation that we term the \emph{spherical \ours gradient flow}.
The term spherical is used to emphasize that the flow stays within the probability measure, as in the spherical Hellinger distance~\citep{LasMie19GPCA}.

To this end, we must first study \emph{spherical MMD} flows over probability measures.
Recall that \eqref{eq:mmd-gfe} is a Hilbert space gradient flow (see \citep{ambrosio2008gradient}) and does not stay within the probability space.
Closely related, many works using kernel-mean embedding~\citep{smolaHilbertSpaceEmbedding2007,muandetKernelMeanEmbedding2017}
also suffer from this issue of not respecting the probability space.
To produce a restricted (or projected) flow in $\cal P$,
our starting point is the \emph{MMD minimizing-movement scheme} restricted to the probability space

\begin{minipage}[c]{0.59\textwidth}
    \begin{align}
        \label{eq:mmd-mms}
        \mu^{k+1}
        \gets\argmin_{\mu\in\cal P} F(\mu ) + \frac1{2\eta}{\mmd}^2(\mu, \mu^{k})
        .
    \end{align}
\end{minipage}
\begin{minipage}[c]{0.39\textwidth}
    \begin{center}
        \includegraphics[width=\textwidth]{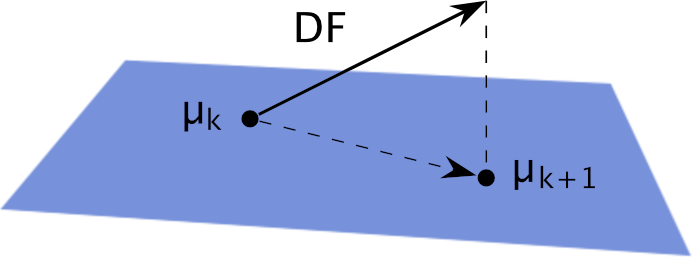}
      \end{center}
\end{minipage}

We now derive the following mass-preserving spherical gradient flows for the MMD and \ours.
\begin{proposition}
    [Spherical MMD and spherical \ours gradient flow equations]
    The spherical MMD gradient flow equation is 
    given by (where $1$ denotes the constant scalar)
    \begin{align}
        \label{eq:spherical-MMD-gfe}
        \dot \mu 
    = - \K^{-1}\left(\dFdmu - \frac{\int \K^{-1}\dFdmu}{\int \K^{-1}1}\right)
        .
    \end{align}
    
    Consequently, 
    the spherical \ours gradient flow equation is
    given by
    \begin{align}
        \label{eq:spherical-ours-gfe}
        \dot \mu =  \alpha\cdot \operatorname{div}(\mu \nabla\dFdmu) 
        -   
        \beta\cdot 
        \K^{-1}\left(\dFdmu - \frac{\int \K^{-1}\dFdmu}{\int \K^{-1}1}\right)
        .
    \end{align}
    Furthermore, those equations are mass-preserving, \ie $\int \dot \mu = 0$.
    \label{prop:spherical-ours-gfe}
\end{proposition}

So far, we have identified the gradient flow equations of interest.
Now, we are ready to present our main theoretical results on the convergence of the \ours gradient flow via functional inequalities.
For example, the logarithmic Sobolev inequality (LSI)
\begin{align}
    \bigg\|{\nabla \log \frac{\dd\mu}{\dd\pi}}\bigg\|^2_{L^2(\mu)}\geq c_{\textrm{LSI}}\cdot \mathrm{D}_\mathrm{KL}(\mu|\pi) \textrm{ for some } c_{\textrm{LSI}}>0
    \label{eq:LSI}
    \tag{LSI}
\end{align}
is sufficient to guarantee the convergence of the pure Wasserstein gradient flow of the KL divergence energy, which governs the same dynamics as the Langevin
equation.
The celebrated \BE Theorem~\citep{bakryDiffusionsHypercontractives1985}, is a cornerstone of convergence analysis for dynamical systems as it provides an explicit
sufficient condition: suppose the target measure $\pi$ is $\lambda$-log concave for some $\lambda>0$, then the global convergence is guaranteed, \ie
\begin{align*}
    \pi=e^{-V}\dd x \textrm{ and } \nabla^2 V\geq \lambda\cdot \ID
    \implies \eqref{eq:LSI} \textrm{ with } c_{\textrm{LSI}}=2\lambda
    \implies \textrm{glob. exp. convergence}
    .
\end{align*}
The question we answer below is whether the \ours gradient flow enjoys such favorable properties.
Our starting point is the (Polyak-)\Loj type functional inequality.
\begin{theorem}
    Suppose the
    following
    \Loj type inequality holds
    for some $c>0$,
    \begin{align}
\alpha\cdot        \bigg\|\nabla \dFdmu\bigg\|^2_{L^2_{\mu}}
        +  
        \beta\cdot \bigg\|{\dFdmu}\bigg\|^2_\rkhs\geq c\cdot \left(F(\mu(t)) - \inf_\mu F(\mu)\right)
        \tag{\ours-\L{}oj}
        \label{eq:loj-ours}
    \end{align}
    for the \ours gradient flow, or
    \begin{align}
        \alpha\cdot        \bigg\|\nabla \dFdmu\bigg\|^2_{L^2_{\mu}}
                +  
                \beta\cdot \bigg\|{\dFdmu- \frac{\int \K^{-1}\dFdmu}{\int \K^{-1}1}}\bigg\|^2_\rkhs\geq c\cdot \left(F(\mu(t)) - \inf_\mu F(\mu)\right)
                \tag{S\ours-\L{}oj}
                \label{eq:loj-ours-sph}
            \end{align}
            for the spherical \ours gradient flow.
            Then,
    the energy $F$
    decays exponentially along the
    corresponding 
    gradient flow, \ie $F(\mu(t)) - \inf_\mu F(\mu)\leq e^{-ct} \cdot \left(F(\mu(0)) - \inf_\mu F(\mu)\right)$.
    \label{prop:loj-ours}
\end{theorem}
To understand a specific gradient flow,
one must delve into the detailed analysis of
the conditions under which
the functional inequalities hold instead of assuming them to hold by default.
We provide such analysis for the \ours gradient flows next.
\subsection{Global exponential convergence analysis}
\label{sec:mmd-energy-convergence}

\paragraph*{MMD energy functional}
As discussed in the introduction, the MMD energy has been proposed as an alternative to the KL divergence energy for sampling by \citet{arbel_maximum_2019}\footnote{
    We do not use the terminology ``MMD gradient flow'' from \citep{arbel_maximum_2019} since it is inconsistent with the naming convention of ``Wasserstein gradient flow'' as Wasserstein refers to the geometry, not the functional.
    \label{footnote:mmd-flow-name}
}, where
they assume the access to samples from the target measure $y_i\sim \pi$.
However, the theoretical convergence guarantees under the MMD energy is a less-exploited topic.
Those authors characterized a local decay behavior under the assumption
that the $\mu_t$ must already be close to the target measure $\pi$ for all $t>0$.
The assumptions they made are not only restrictive, but also difficult to check.
There has also been no global convergence analysis. For example, \citet{arbel_maximum_2019}'s Proposition~2 states that the MMD is non-increasing, which is not equivalent to convergence and is easily satisfied by other flows. The mathematical limitation is that the MMD is in general not guaranteed to be convex along the Wasserstein geodesics.
In addition, our analysis also does not require the heuristic noise injection step as was required in their implementation.
\citet{mroueh_unbalanced_2020}
also used the MMD energy but with a different gradient flow,
which
has been shown by \citet{zhu2024approximation}
to be a kernel-regularized
inf-convolution of the Allen-Cahn and Cahn-Hilliard type of dissipation.
However, \citet{mroueh_unbalanced_2020}'s
convergence analysis is not sufficient for establishing (global) exponential convergence as no functional inequality has been established there.
In contrast, we now provide full global exponential convergence guarantees.

We first provide an interesting property that will become useful for our analysis.
\begin{theorem}
    Suppose the driving energy is the squared MMD,
    $F(\mu)=\frac12\mmd^2(\mu,\pi)$
    and initial datum $\mu_0\in\mathcal P$ is a probability measure.
    Then, the spherical MMD gradient flow equation~\eqref{eq:spherical-MMD-gfe} coincides with the MMD gradient flow equation
    \begin{align}
        \label{eq:spherical-MMD-MMD-gfe}
        \dot \mu &= - \left(\mu-\pi\right),
         \tag{MMD-MMD-GF}
\end{align}
whose solution is 
a linear interpolation between the initial measure $\mu_0$ and the target measure $\pi$, \ie 
$$\mu_t = e^{-t}\mu_0 + (1-e^{-t})\pi.$$

    Furthermore, same coincidence holds for the spherical \ours and \ours gradient flow equation
    \begin{align}
    \label{eq:spherical-MMD-IFT-gfe}
        \dot \mu &= \alpha\cdot \operatorname{div}\left(\mu \int \nabla_2 k(x,\cdot )\dd \left(\mu-\pi\right)(x)\right)
        - \beta \left(\mu-\pi\right)
        .
        \tag{MMD-\ours-GF}
    \end{align}
    \label{thm:spherical-MMD-MMD-gfe}
\end{theorem}

The explicit solution to the ODE~\eqref{eq:spherical-MMD-MMD-gfe} shows an exponential convergence to the target measure $\pi$ along the (spherical) MMD gradient flow.
The (spherical) \ours gradient flow equation~\eqref{eq:spherical-MMD-IFT-gfe}
differs
from the Wasserstein gradient flow equation of \citep{arbel_maximum_2019} by a linear term.
This explains the intuition of why we can expect good convergence properties for the \ours gradient flow of the squared MMD energy.
We exploit this feature of the \ours gradient flow
to show global convergence
guarantees
for inference with the MMD energy.
This has not been possible previously when confined to the pure
Wasserstein gradient flow.
\begin{theorem}
    [Global exponential convergence of
    the \ours flow of the MMD energy]
    Suppose the energy $F$ is the squared MMD energy $F(\mu)=\frac12\mmd^2(\mu, \nu)$.
    Then, the \eqref{eq:loj-ours} holds globally
    with a constant $c\geq 2 \beta>0$.
    
    Consequently, for any initialization within the non-negative measure cone
     $\mu_0\in\Mplus$, the squared MMD energy decays exponentially along the \ours gradient flow
     of non-negative measures, \ie
    \begin{align}
            \frac12\mmd^2(\mu_t, \nu)
        \leq 
        e^{- 2 \beta t}\cdot  \frac12\mmd^2(\mu_0, \nu)
        .
        \label{eq:mmd-energy-decay}
    \end{align}

    Furthermore, if the initial datum $\mu_0$
    and the target measure $\pi$
    are probability measures $\mu_0, \pi \in\cal P$,
    then the squared MMD energy decays exponentially globally
    along the spherical \ours gradient flow,
    \ie 
    the decay estimate~\eqref{eq:mmd-energy-decay}
    holds along the spherical \ours gradient flow
    of probability measures.
    \label{thm:mmd-energy-decay}
\end{theorem}
We emphasize that no \BE type or kernel conditions are required -- the \Loj inequality holds globally when using the \ours flow.
In contrast, the Wasserstein-Fisher-Rao gradient flow
\begin{align}
    \dot \mu =  
    \alpha\cdot \operatorname{div}\left(\mu \int \nabla_2 k(x,\cdot )\dd \left(\mu-\pi\right)(x)\right)
    - \beta \cdot
    \int k(x, \cdot )
    \dd
    \left(
        \mu -{ \pi}
    \right)(x)
    \tag{MMD-WFR-GF}
    \label{eq:wfr-MMD-gfe}
\end{align}
does not enjoy such global convergence guarantees.
\paragraph*{KL energy functional}
For variational inference and MCMC, a common choice of the energy is the KL divergence energy, \ie $F(\mu)=\mathrm{D}_\mathrm{KL}(\mu|\pi)$.
This has already been studied by a large body of literature, including the case of Wasserstein-Fisher-Rao~\cite{LiMiSa23FPGG,luAcceleratingLangevinSampling2019}.
Not surprisingly, \eqref{eq:LSI} is still sufficient for the exponential convergence of the WFR type of gradient flows since the dissipation of the Wasserstein part alone is sufficient for driving the system to equilibrium.
For the \ours gradient flows under the KL divergence energy functional, the convergence can still be established.
This showcases the strength of the \ours geometry -- it enjoys the best of both worlds.
The \ours gradient flow equation of the KL divergence energy reads
    $\displaystyle \dot \mu = \alpha\cdot \operatorname{div}(\mu \nabla \log \frac{\dd\mu}{\dd\pi}) - 
        \beta\cdot \K^{-1}  \log \frac{\dd\mu}{\dd\pi}$.
Unlike the MMD-energy flow case, the spherical \ours gradient flow of the KL over probability measures $\cal P$ no longer coincides with that of the (non-spherical) \ours and is
given by
\begin{align}
    \dot \mu = \alpha\cdot \operatorname{div}(\mu \nabla \log \frac{\dd\mu}{\dd\pi}) - 
    \beta\cdot \K^{-1}  
    \left(
        \log \frac{\dd\mu}{\dd\pi} - \frac{\int \K^{-1}\log \frac{\dd\mu}{\dd\pi}}{\int \K^{-1}1}
        \right)
    .
    \label{eq:kl-gfe-ift-sph}
\end{align}
\begin{proposition}
    [Exponential convergence of the S\ours gradient flow of the KL divergence energy]
    Suppose the \eqref{eq:LSI} holds
    with $c_{\textrm{LSI}}=2\lambda$ or 
    the target measure $\pi$ is $\lambda$-log concave for some $\lambda>0$.
    Then, the KL divergence energy decays exponentially globally along
    the spherical \ours gradient flow~\eqref{eq:kl-gfe-ift-sph}
    ,\ie
$\displaystyle\mathrm{D}_\mathrm{KL}(\mu_t|\pi) \leq e^{- 2\alpha \lambda t} \mathrm{D}_\mathrm{KL}(\mu_0|\pi)$.
\label{prop:kl-energy-decay}
\end{proposition}
The intuition behind the above result is that the S\ours gradient flow
converges whenever the pure Wasserstein gradient flow, \ie its convergence is at least as fast as the Wasserstein gradient flow.
However, we emphasize that the decay estimate of the KL divergence energy only holds along the spherical \ours flow over probability measures $\cal P$, but not the full \ours flow over non-negative measures $\Mplus$.

\subsection{Minimizing movement, JKO-splitting, and a practical particle-based algorithm}
In applications to machine learning and computation,
continuous-time flow can be discretized via
the JKO
scheme~\citep{jordan_variational_1998}, which is based on the minimizing movement scheme (MMS)~\citep{de1993new}.
For the reaction-diffusion type
gradient flow equation in
the Wasserstein-Fisher-Rao setting,
the
\emph{JKO-splitting} a.k.a. \emph{time-splitting} scheme
has been studied by
\citet{gallouet2017jko,mielkeTimesplittingMethodsGradient2023}.
This amounts to
splitting the 
diffusion (Wasserstein)
and
reaction (MMD) step in \eqref{eq:ikw-gfe-unreg},
\ie at time step $\ell \geq 1$
    \begin{align}
        \label{eq:jko-split}
        \begin{aligned}
            \mu^{\ell+\frac12}
            &\gets\argmin_{\mu\in\cal P} F(\mu ) + \frac1{2\tau}W_2^2(\mu, \mu^\ell)
            ,
            && \textrm{(Wasserstein step)}
            \\
            \mu^{\ell+1}
            &\gets\argmin_{\mu\in\cal P} F(\mu ) + \frac1{2\eta}{\mmd}^2(\mu, \mu^{\ell+\frac12})
            .
            && \textrm{(MMD step)}
        \end{aligned}
    \end{align}
A similar JKO-splitting scheme can also be constructed via the WFR gradient flow,
which amounts to replacing the MMD step in
\eqref{eq:jko-split}
with a proximal step in the KL (as an approximation to the Hellinger), \ie 
$\mu^{\ell+1}
\gets\argmin_{\mu\in\cal P} F(\mu ) + \frac1{\eta}{\mathrm{D}_{\mathrm{KL}}}(\mu| \mu^{\ell+\frac12})$,
which is well-studied in the optimization literature
as the entropic mirror descent~\citep{nemirovskijProblemComplexityMethod1983}. %
Our MMD step can also be viewed as a mirror descent step with the mirror map $\frac12\|\K \cdot \|^2_{\rkhs}$.
However, for the task of MMD inference of \citep{arbel_maximum_2019}, WFR flow does not possess convergence guarantees such as our Theorem~\ref{thm:spherical-MMD-MMD-gfe}. 
The MMD step can also be easily implemented as in our simulation.

We summarize the resulting overall \ours particle gradient descent from the JKO splitting scheme in Algorithm~\ref{alg:jko-split} in the appendix.
We now look at those two steps 
respectively.
For concreteness,
we consider a flexible particle approximation to the probability measures, with possibly non-uniform weights allocated to the particles, \ie
$\mu = \sum_{i=1}^{n}\alpha_i\delta_{x_i}, \alpha \in \Delta^n, \ x_i\in\XX$.
    \paragraph{Wasserstein step: particle position update.}
    \eqref{eq:jko-split} is a standard JKO scheme;
    see, e.g., \cite{santambrogio_optimal_2015}.
    The optimality condition of the Wasserstein proximal step can be
    implemented using
    a particle gradient descent algorithm
    \begin{align}
        x^{\ell+1}_i = x^{\ell}_i - {\tau} \cdot \nabla \frac{\delta F}{\delta \mu}[\mu^\ell] (x^{\ell}_i), \ i=1,...,n,
        \label{eq:wasserstein-step-particle-update}
    \end{align}
    which is essentially the algorithm proposed by \citet{arbel_maximum_2019}
    when 
    $F(\mu)=\frac12\mmd^2(\mu,\pi)$.
    \paragraph{MMD step: particle weight update.}
    The MMD step in \eqref{eq:jko-split} is a discretization step of the spherical MMD gradient flow, as shown in \eqref{eq:mmd-mms} and Proposition~\ref{prop:spherical-ours-gfe}.
    We propose to use the updated particle location $x_i^{\ell+1}$ from the Wasserstein step~\eqref{eq:wasserstein-step-particle-update} and update the weights $\beta_i$ by solving
    \begin{align*}
        \inf_{\beta\in \Delta ^n} 
        F( \sum_{i=1}^{n}\beta_i\delta_{x_i^{\ell+1}})
            +
            \frac1{2\eta}
        \operatorname{MMD}^2(  \sum_{i=1}^{n}\beta_i\delta_{x^{\ell+1}_i},  \sum_{i=1}^{n}\alpha^\ell_i\delta_{x^{\ell+1}_i})
        ,
    \end{align*}
    \ie the MMD step only updates the weights.
    Alternatively, as in the classical mirror descent optimization,
    one can use a linear approximation
    $
    F(\mu) \approx F(\mu^\ell) + \langle \dFdmuk, \mu - \mu^\ell\rangle
    _{L^2}
    $.
We also provide a specialized discussion on
the MMD-energy minimization task of \citep{arbel_maximum_2019}.
Let the energy objective be the squared MMD
$%
    F(\mu ) := \frac{1}{2}\mmd(\mu, \pi)^2$.
In this setting, we are given the particles sampled from the target measure $y^i\sim \pi$.
For the MMD step in \eqref{eq:jko-split}, the computation is drastically simplified
to an MMD barycenter problem,
which was also studied in \citep{cohenEstimatingBarycentersMeasures2021}.
   This amounts to solving
   a convex quadratic program with a simplex constraint; see the appendix for the detailed expression.
\section{Numerical Example}
\label{sec:numerical-example}
The overall goal of the numerical experiments
is to approximate the target measure $\pi$ by minimizing the squared MMD energy, \ie
$\displaystyle\min_{\mu\in A\subset \cal P}
           \mmd^2 (\mu,  \pi)$.
In all the experiments, we have access to the target measure $\pi$ in the form of samples $y_i\sim \pi$.
This setting was studied in \citep{arbel_maximum_2019} as well as in many deep generative model applications.
In the following experiments, we compare the performance of our proposed
algorithm
of \ours gradient flow,
which implements the JKO-splitting scheme in
\eqref{eq:jko-split} and is detailed in
Algorithm \ref{alg:jko-split},
to that of
\emph{(1)}
\citet{arbel_maximum_2019}'s
the ``MMD flow'' (see our discussion in \footref{footnote:mmd-flow-name}),
we used their algorithm both with and without a heuristic noise injection suggested by those authors;
\emph{(2)}
the Wasserstein-Fisher-Rao flow of the MMD \eqref{eq:wfr-MMD-gfe}.
The WFR flow was also used by \citet{yanLearningGaussianMixtures2023,lu2023birth} but for minimizing the KL divergence function.
As discussed in \S\ref{sec:mmd-energy-convergence}, the MMD flow
of \citep{arbel_maximum_2019} does not possess global convergence guarantees while \ours does.
Furthermore, in the Gaussian mixture target experiment, the target measure $\pi$ is not log-concave.
We emphasize that our convergence guarantee still holds for the \ours flow while there is no decay guarantee for the WFR flow.
We provide the code for the implementation at \url{https://github.com/egorgladin/ift_flow}.

\begin{figure}[ht]
    \centering
    \begin{subfigure}[b]{0.49\textwidth}
        \centering
        \includegraphics[width=\textwidth]{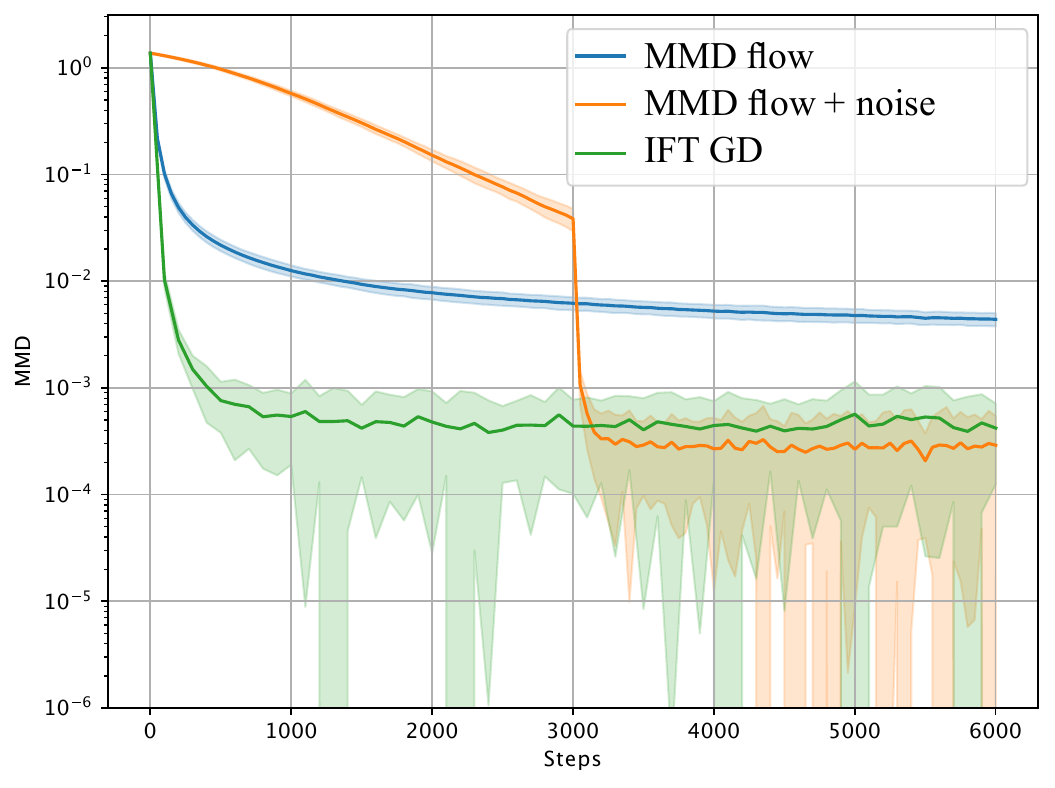}
        \caption{Gaussian target experiment}
        \label{fig:image1}
    \end{subfigure}
    \hfill
    \begin{subfigure}[b]{0.49\textwidth}
        \centering
        \includegraphics[width=\textwidth]{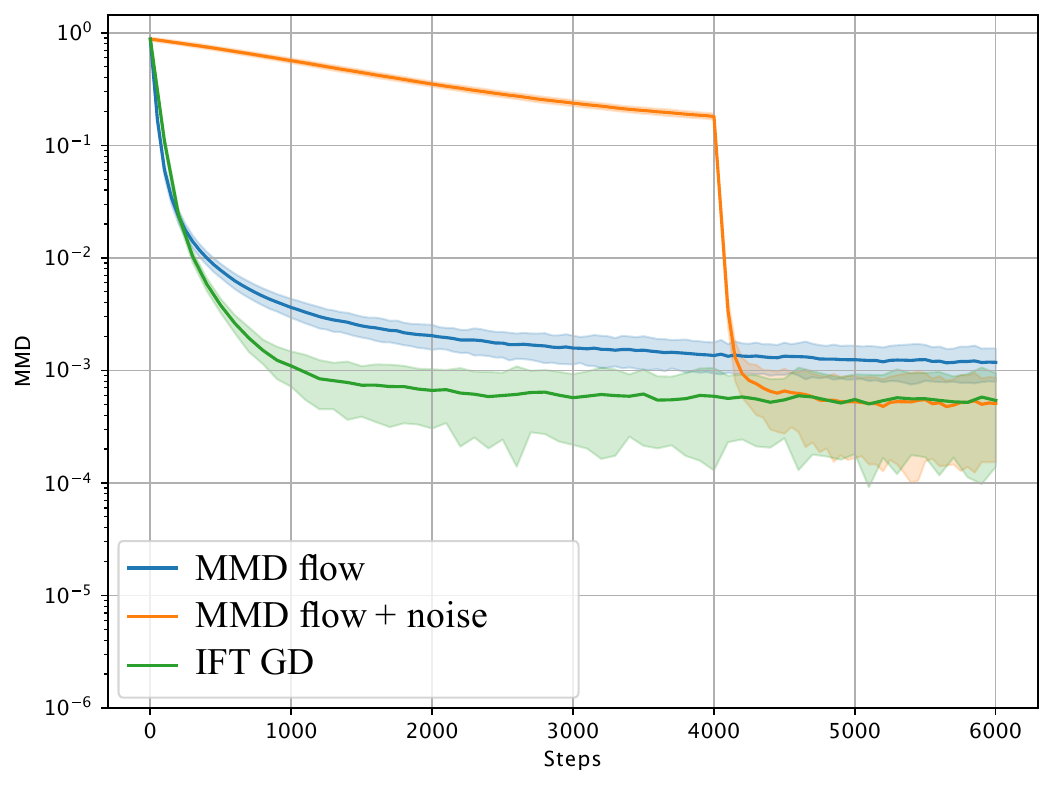}
        \caption{Gaussian mixture target experiment}
        \label{fig:image2}
    \end{subfigure}
    \caption{Mean loss and standard deviation computed over 50 runs}
    \label{fig:side_by_side}
\end{figure}

\begin{figure}
    \centering
    \begin{subfigure}{.32\textwidth}
        \centering
        \includegraphics[width=.95\linewidth]{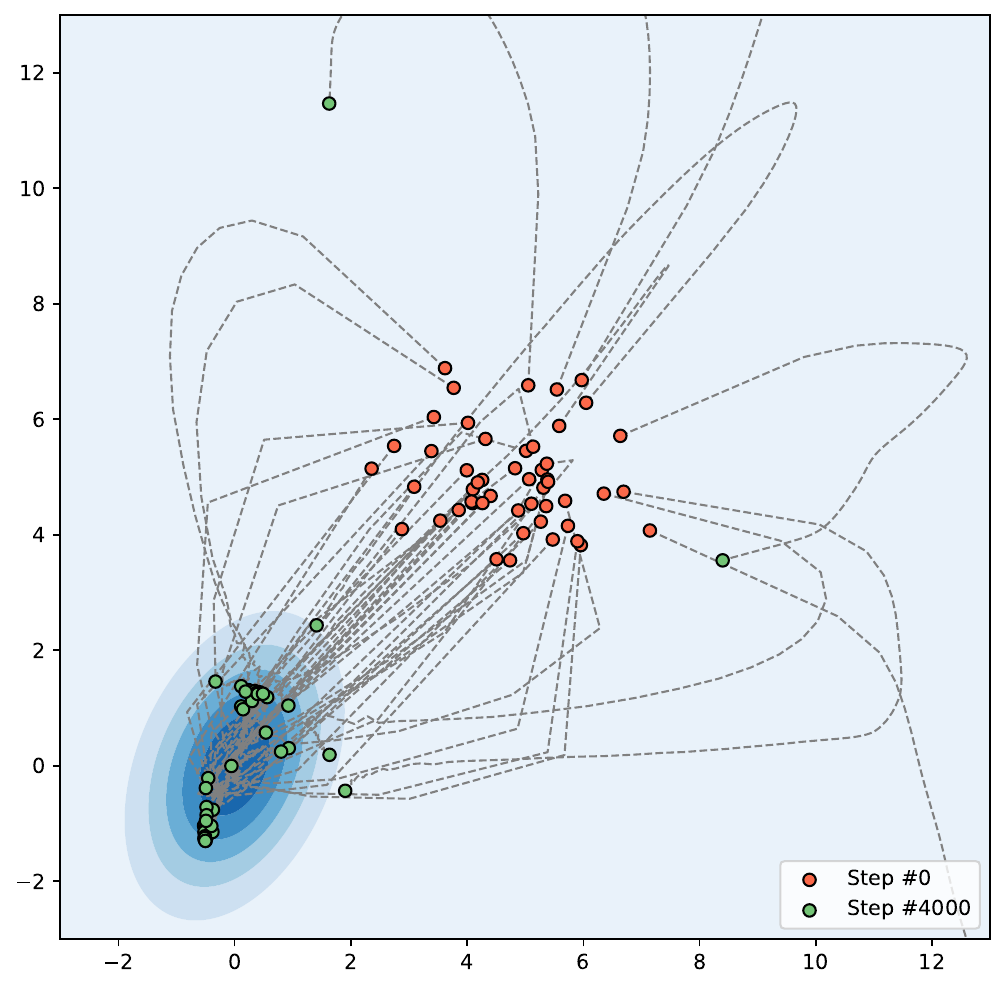}  
        \caption{Vanilla MMD flow}
    \end{subfigure}
    \begin{subfigure}{.32\textwidth}
        \centering
        \includegraphics[width=.95\linewidth]{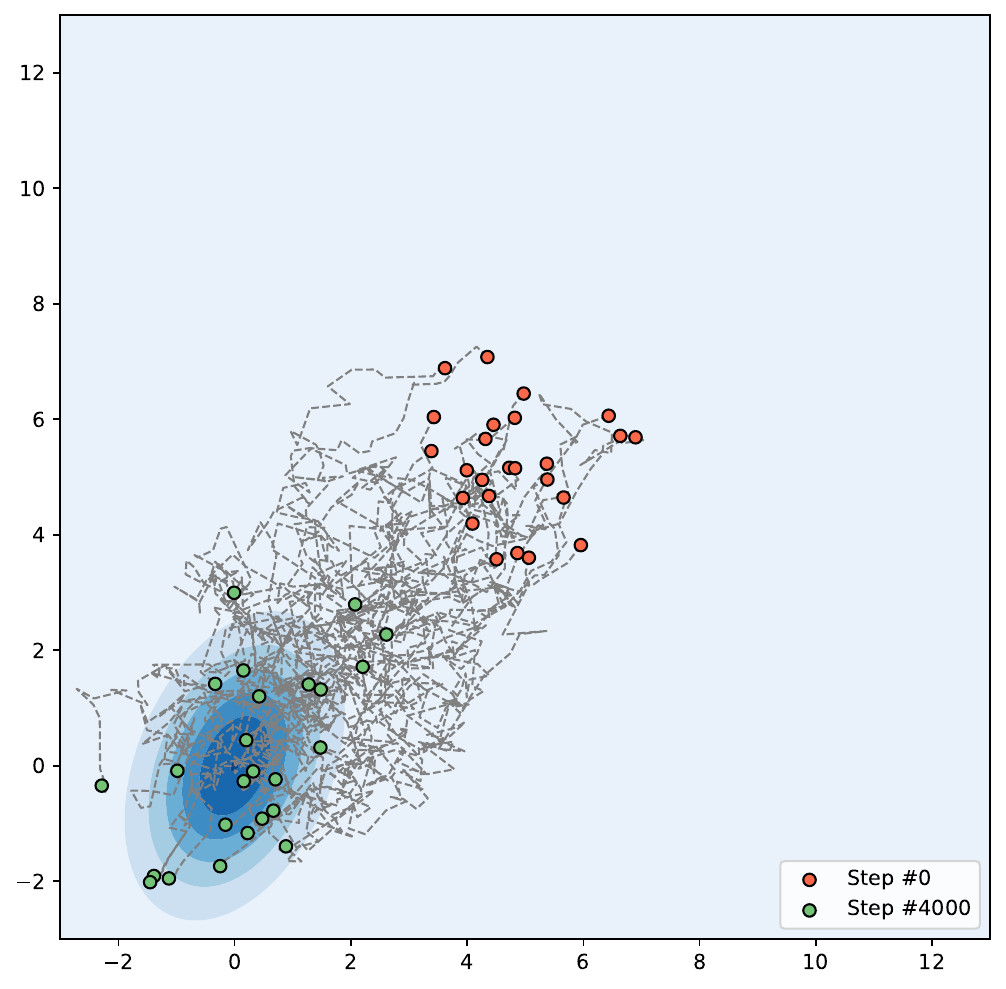}  
        \caption{MMD flow + noise}
    \end{subfigure}
    \begin{subfigure}{.32\textwidth}
        \centering
        \includegraphics[width=.95\linewidth]{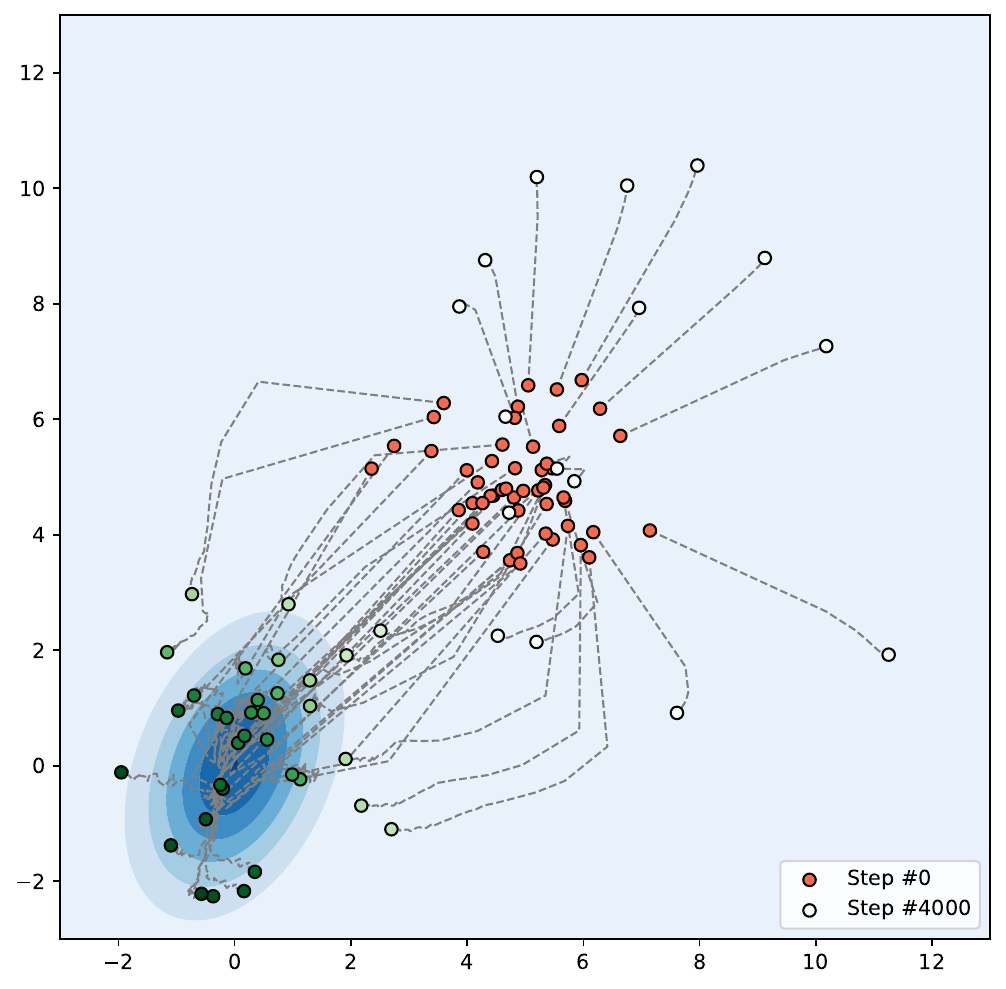}  
        \caption{\ours (Algorithm \ref{alg:jko-split})}
    \end{subfigure}
    \caption{Trajectory of a randomly selected subsample produced by different algorithms in the Gaussian target experiment. Color intensity indicates points' weights.
    The hollow dots indicate the particles that have already vanished.}
    \label{fig:traj1}
\end{figure}
\begin{figure}
    \centering
    \begin{subfigure}{.32\textwidth}
        \centering
        \includegraphics[width=.95\linewidth]{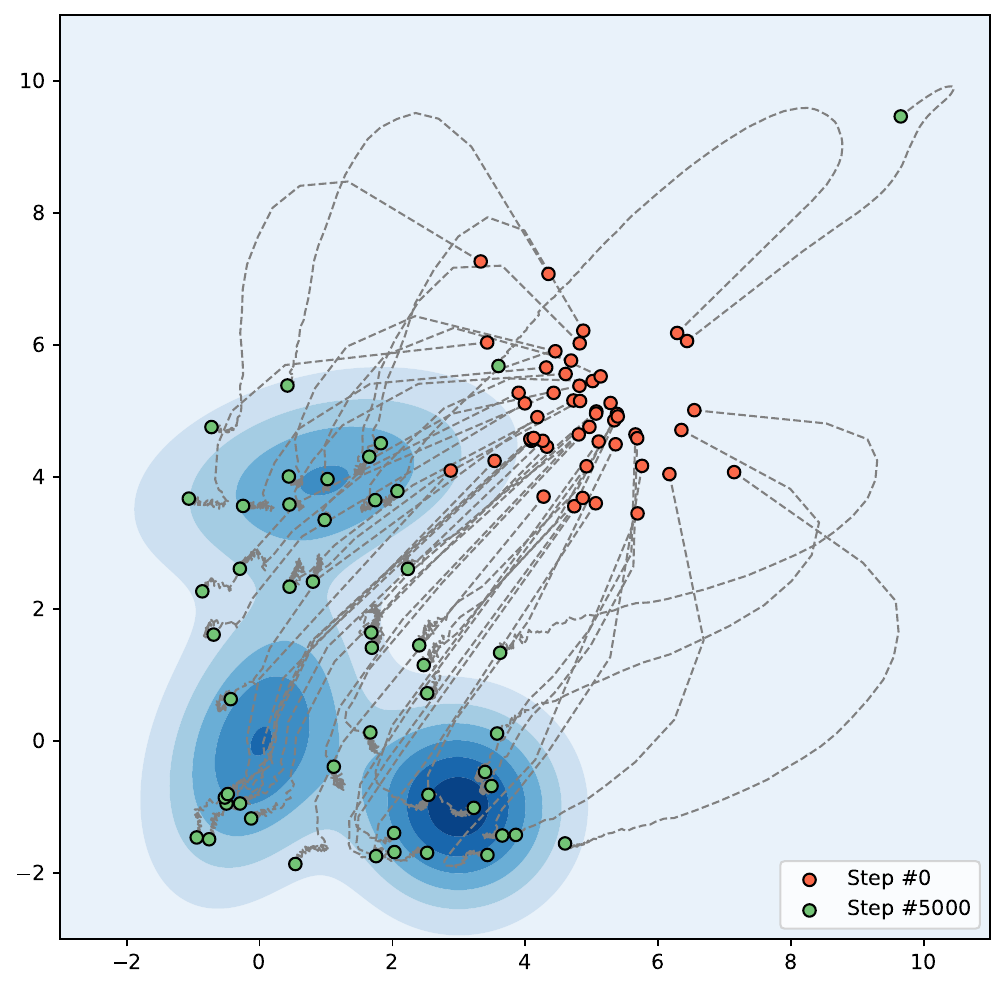}  
        \caption{Vanilla MMD flow}
    \end{subfigure}
    \begin{subfigure}{.32\textwidth}
        \centering
        \includegraphics[width=.95\linewidth]{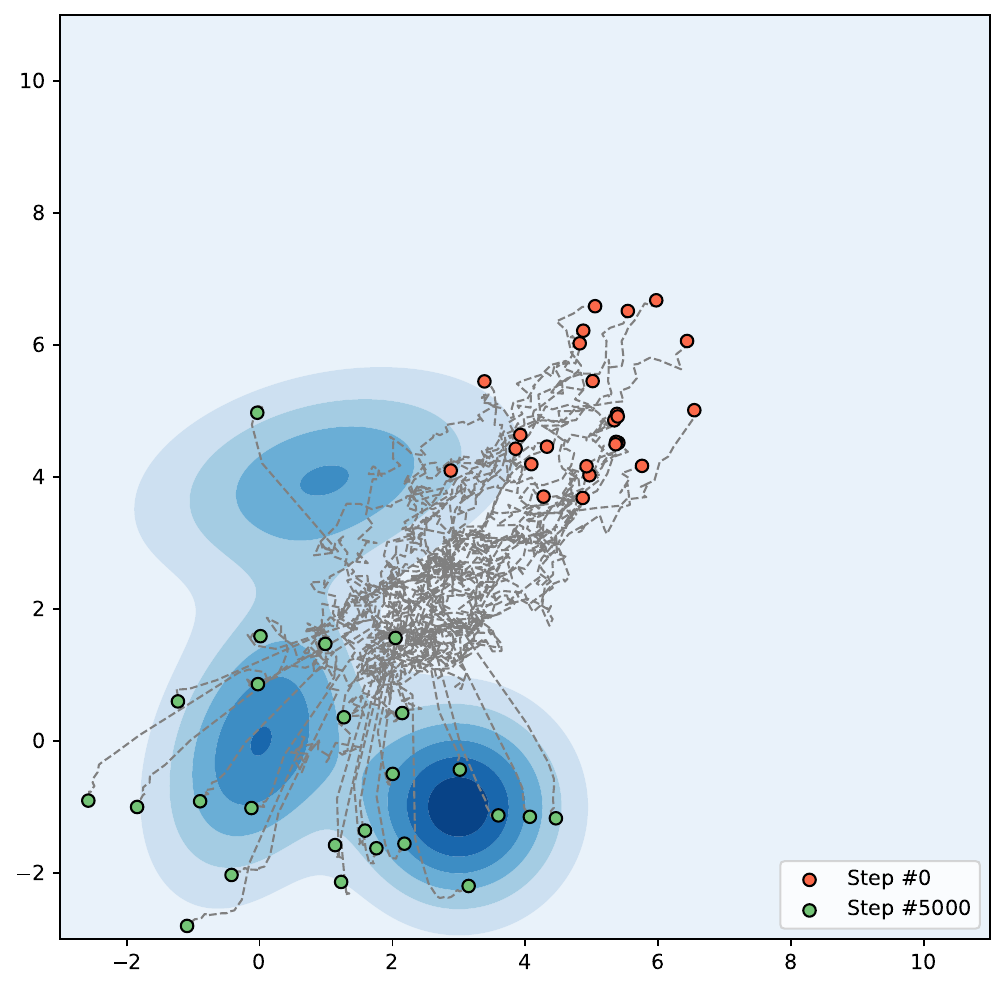}  
        \caption{MMD flow + noise}
    \end{subfigure}
    \begin{subfigure}{.32\textwidth}
        \centering
        \includegraphics[width=.95\linewidth]{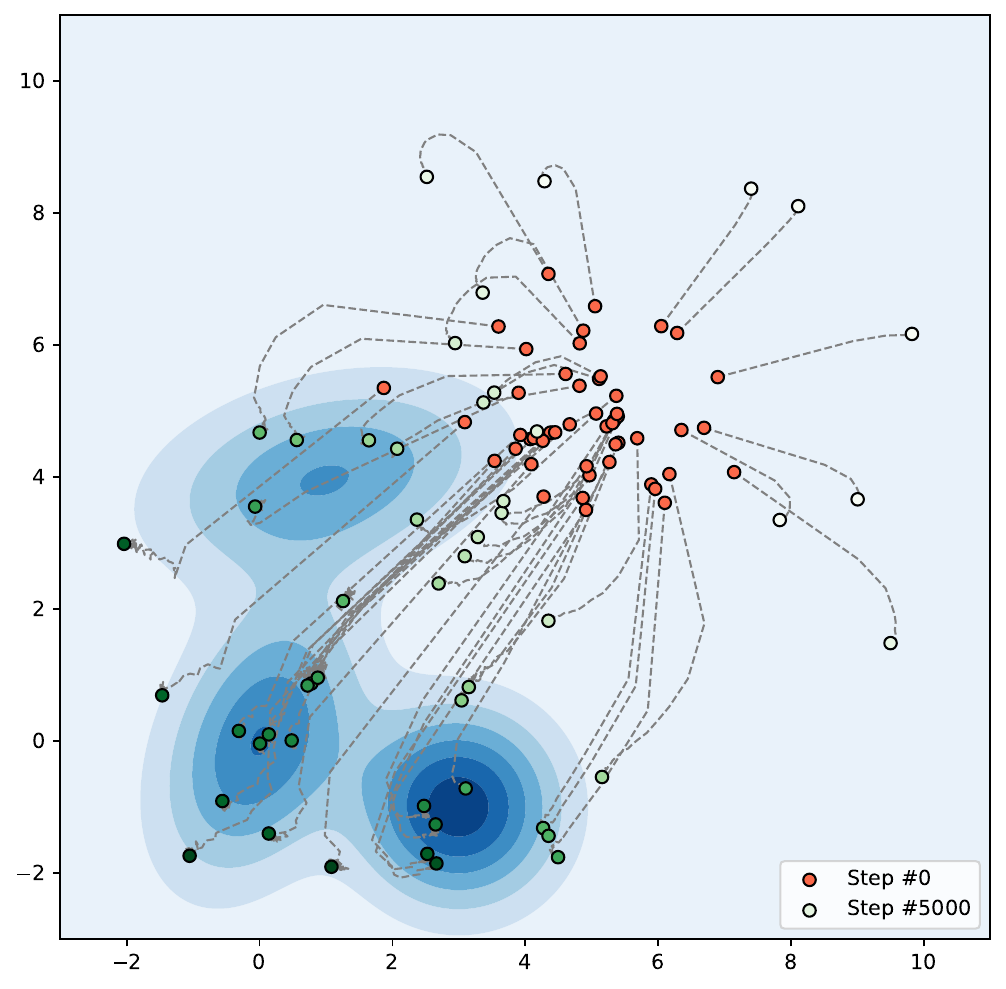}  
        \caption{\ours (Algorithm \ref{alg:jko-split})}
    \end{subfigure}
    \caption{Trajectory of a randomly selected subsample produced by different algorithms in the Gaussian mixture experiment. Color intensity indicates points' weights.
    The hollow dots indicate the particles that have already vanished.}
    \label{fig:traj2}
\end{figure}
\textbf{Gaussian target in 2D experiment}
Figures \ref{fig:side_by_side}(a) and \ref{fig:traj1} showcase the performance of the algorithms in a setting where $\mu^0$ and $\pi$ are both Gaussians in 2D. Specifically, $\mu^0 \sim \mathcal{N}(5 \cdot \mathbf{1}, I)$ and $\pi \sim \mathcal{N}\left(\mathbf{0}, \SmallMatrix{1&1/2\\1/2&2}\right)$. The number of samples drawn from $\mu^0$ and $\pi$ was set to $n=100$. 
A Gaussian kernel with bandwidth $\sigma=10$ was used.
For all three algorithms, we chose the largest stepsize that didn't cause unstable behavior, $\tau=50$. The parameter $\eta$ in \eqref{mmd_st} was set to 0.1.
As can be observed from the trajectories produced by MMD flow (Figure \ref{fig:traj1}(a)), most points collapse into small clusters near the target mode. Some points drift far away from the target distribution and get stuck; the resulting samples represent the target distribution poorly,
which is a sign of suboptimal solution. 
MMD flow with
the heuristic noise injection
produces much better results. 
We suspect the noise helps to escape local minima;
however, the injection needs to be heuristically tuned.
However, it takes a large number iterations for points to get close to locations with high density of the target distribution.
Similarly to the previous research on noisy MMD flow, we use a relatively large noise level (10) in the beginning and ``turn off'' the noise after a sufficient number of iterations (3000 in our case). A drawback of this approach is that the right time for noise deactivation depends on the particular problem instance, which makes the algorithm behavior less predictable.
\begin{wrapfigure}[]{t}{0.54\textwidth}
    \centering
    \includegraphics[width=0.5\textwidth]{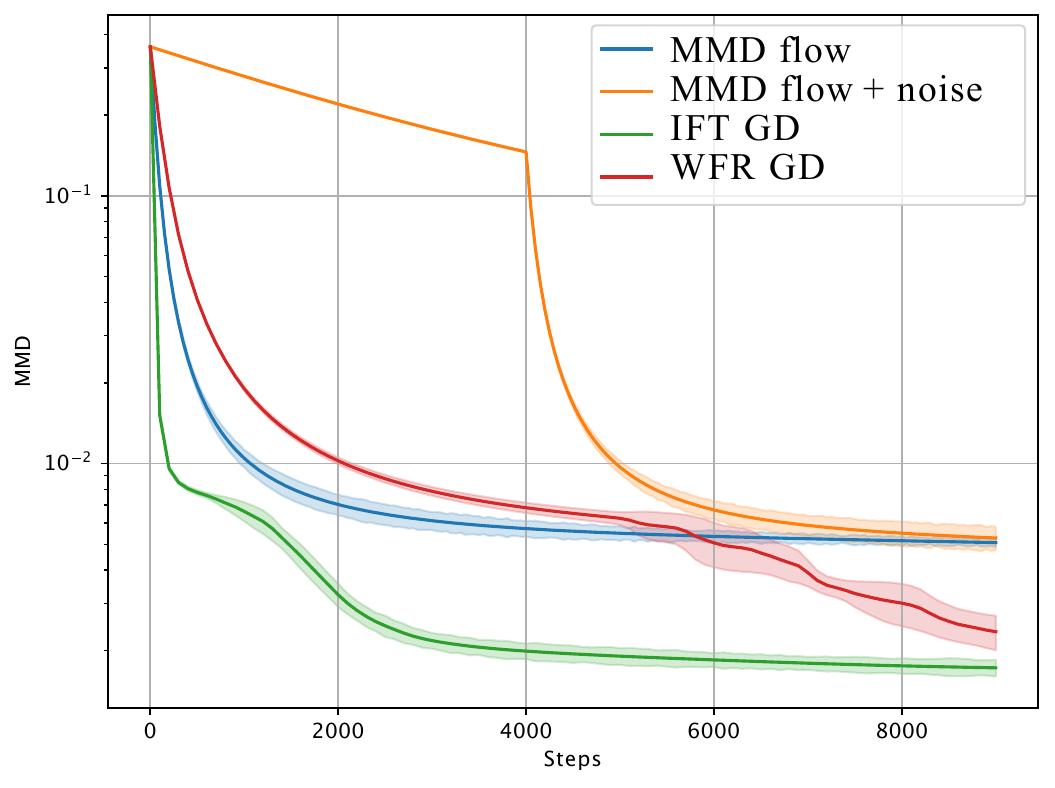}
    \caption{Comparison with the WFR flow of the MMD in 100 dimensions}
    \label{fig:image2-wfr}
\end{wrapfigure}
Algorithm \ref{alg:jko-split} achieves a similar accuracy to that of the noise-injected MMD flow, but much faster -- already after 1000 steps -- without 
any heuristic noise injection.
For the few particles that did not make it close to the target Gaussian's mean,
their mass is teleported to those particles that are close to the target.
Hence, the resulting performance of the \ours algorithm does not deteriorate.
The hollow dots in the trajectory plot indicate the particles whose mass has been teleported and hence their weights are zero.
This is a major advantage of unbalanced transport for dealing with local minima.
For a faster implementation, 
in the implementation of the MMD step in \eqref{eq:jko-split}, we only perform a single step of projected gradient descent instead of computing a solution to the auxiliary optimization problem \eqref{mmd_st}. To be fair in comparison, we count each iteration as two steps. Thus, 6000 steps of the algorithm in Figure \ref{fig:side_by_side}(a) correspond to only 3000 iterations, \ie the results for \ours algorithm have already been handicapped by a factor of $2$.
We would also like to note that Algorithm \ref{alg:jko-split} was executed with constant hyperparameters without further tuning over the iterations, in contrast to the noisy MMD flow.
In practical implementations, it is possible to further improve 
the performance
by sampling new locations (particle rejuvenation) in the MMD step~\eqref{eq:jko-split} as done similarly in \cite{daiProvableBayesianInference2016}.
Since this paper is a theoretical one and not about 
competitive benchmarking, we leave this for future work.

\textbf{Gaussian mixture in 2D experiment}
The second experiment has a similar setup. However, this time the target is a mixture of equally weighted Gaussian distributions,
\begin{equation*}
    \mathcal{N}\left(\mathbf{0}, \SmallMatrix{1&1/2\\1/2&2}\right),\quad \mathcal{N}\left(\SmallMatrix{3\\-1}, I\right),\quad \mathcal{N}\left(\SmallMatrix{1\\4}, \SmallMatrix{3&1/2\\1/2&1}\right).
\end{equation*}
Figures \ref{fig:side_by_side}(b) and \ref{fig:traj2} showcase loss curves and trajectories produced by the considered algorithms.

\textbf{WFR flow for Gaussian mixture target in 100D}
We conducted an experiment in dimension $d=100$, comparing the \ours flow with the WFR flow of the MMD energy. The initial distribution is $\mathcal{N}(0, I)$, and the target $\pi$ is a mixture of 3 distributions $\mathcal{N}(m_i, \Sigma_i),\, i=1,2,3$, where $m_i$ and $\Sigma_i$ are randomly generated such that $\|m_i\|_2=20$ and the smallest eigenvalue of $\Sigma_i$ is greater than $0.5$.
For fairness,
each iteration of IFT particle GD (as well as its version with KL step) counts as two steps, i.e., these methods only performed 4500 iterations. In the noisy MMD flow, the noise is disabled after 4000 steps. All methods are used with equal stepsize.

\section{Discussion}
In summary, 
the (spherical) \ours gradient flows are a suitable choice for energy minimization of both the MMD and KL divergence energies with sound global exponential convergence guarantees.
There is also an orthogonal line of works studying the Stein gradient flow and descent~\cite{liu_stein_2019,duncan2019geometry}, which also has the mechanics interpretation of repulsive forces.
It can be related to our work in that the \ours gradient flow has a (de-)kernelized reaction part, while the Stein flow has a kernelized diffusion part.
Furthermore, there is a work \cite{manupriyaMMDRegularizedUnbalancedOptimal2024} that proposes a static MMD-regularized Wasserstein distance,
which should not be confused with our \ours gradient flow geometry.
Another future direction is sampling and inference when we do not have access to the samples from the target distribution $\pi$, but can only evaluate its score function $\nabla \log \pi$.

\appendix
\section{Appendix: proofs and additional technical details}
\label{sec:proofs}

\begin{proof}
    [Proof of Proposition~\ref{prop:spherical-ours-gfe}]
    We first consider the MMS step~\eqref{eq:mmd-mms}.
    The Lagrangian of the MMS step is, for $\lambda\in\mathbb R$,
    \begin{align*}
        \mathcal L(\mu, \lambda) =
        F(\mu ) + \frac1{2\eta}{\mmd}^2(\mu, \mu^{\ell}) +\lambda (\int \mu - 1)
        .
    \end{align*}
    The Euler-Lagrange equation gives
    \begin{align}
        \dFdmuk  + \frac1{\eta}\K (\mu - \mu^{\ell})
        +\lambda&=0,
        \\
        \int \mu  &= 1
        .
        \label{eq:spherical-mmd-mms-EL}
    \end{align}
Rewriting the first equation,
\begin{align*}
    \mu = \mu^{\ell} - \eta \K^{-1}\dFdmuk - \eta\lambda\K^{-1}
    .
\end{align*}
Integrating both sides, we obtain
\begin{align*}
    1 = 1 - \eta \int \K^{-1}\dFdmuk - \eta\lambda\int \K^{-1}1
    \implies
    \lambda= -\frac{\int \K^{-1}\dFdmuk}{\int \K^{-1}1}
    .
\end{align*}
Let the time step $\eta\to 0$ in the first equation in the Euler-Lagrange equation~\eqref{eq:spherical-mmd-mms-EL},
we obtain
\begin{align}
    \dot \mu = - \K^{-1}(\dFdmu + \lambda)
    = - \K^{-1}\dFdmu + \frac{\int \K^{-1}\dFdmu}{\int \K^{-1}1}\cdot \K^{-1}1
    ,
    \label{eq:pf-spherical-mmd-gfe-1}
\end{align}
which is the desired spherical MMD gradient flow equation.
    
    Spherical \ours gradient flow equation is obtained by an inf-convolution~\citep{gallouet2017jko,liero_optimal_2018,chizat_unbalanced_2019} of the above spherical MMD and Wasserstein part.
    The verification of the mass-preserving property is by a straightforward integration of \eqref{eq:pf-spherical-mmd-gfe-1}
    \begin{align*}
        0 = \int \dot \mu = - \int \K^{-1}\dFdmu + \frac{\int \K^{-1}\dFdmu}{\int \K^{-1}1}\cdot \int \K^{-1}1
        =0.
    \end{align*}
    Hence, the theorem is proved.
\end{proof}

\begin{proof}
    [Proof of Theorem~\ref{prop:loj-ours}]
    The proof amounts to identifying the correct left-hand side of the \Loj type inequality.
    We take the time derivative of the energy
    \begin{multline*}
        \frac{\dd }{\dd t}
            F(\mu)
        =
        \langle
            \dFdmu
             ,
             {\alpha}\cdot \operatorname{div}(\mu \nabla\dFdmu) -   {\beta} \cdot \K^{-1} \dFdmu
            \rangle_{L^2}
            \\
            =
-\alpha\cdot \|\nabla \dFdmu\|^2_{L^2_{\mu}}
- \beta\cdot \| \dFdmu\|^2_{\rkhs}
,
    \end{multline*}
which is the desired left-hand side of the \Loj type inequality.

For the spherical \ours gradient flow, we have
\begin{multline*}
    \frac{\dd }{\dd t}
        F(\mu)
    =
    \langle
        \dFdmu
         ,
         {\alpha}\cdot \operatorname{div}(\mu \nabla\dFdmu) 
         -
            \beta\cdot 
        \K^{-1}\left(\dFdmu - \frac{\int \K^{-1}\dFdmu}{\int \K^{-1}1}\right)
        \rangle_{L^2}
        \\
        =
        \langle
            \dFdmu
             ,
             {\alpha}\cdot \operatorname{div}(\mu \nabla\dFdmu) 
             \rangle_{L^2}
             +
             \langle
              \dFdmu - \frac{\int \K^{-1}\dFdmu}{\int \K^{-1}1}
             ,
             -
                \beta\cdot 
            \K^{-1}\left(\dFdmu - \frac{\int \K^{-1}\dFdmu}{\int \K^{-1}1}\right)
            \rangle_{L^2}
            \\
=
-\alpha\cdot \Bigg\|\nabla \dFdmu\Bigg\|^2_{L^2_{\mu}}
- \beta\cdot \Bigg\| \dFdmu- \frac{\int \K^{-1}\dFdmu}{\int \K^{-1}1}\Bigg\|^2_{\rkhs}
,
\end{multline*}
where the second equality follows from the fact that the spherical \ours gradient flow is mass-preserving.
Hence, the left-hand side of the \Loj type inequality is obtained.

\end{proof}

\begin{proof}
    [Proof of Corollary~\ref{cor:kernel-approx-ours}]
    The formal proof is by using
    well-known results for kernel smoothing in non-parametric statistics~\citep{tsybakov_introduction_2009}.
    Note that the
    gradient flow equation
    can be rewritten as
\begin{align*}
    \dot u 
    = - \alpha\cdot \DIV (\mu\nabla \dFdmu) + \beta\cdot  \K ^{-1}\dFdmu
    = - \alpha\cdot \DIV (\mu\nabla \dFdmu) + \beta\cdot \mu \K_\mu ^{-1}\dFdmu
    .
\end{align*}
Recall the definition of the integral operator 
\begin{align}
    \K f = \int k_\sigma(\cdot ,y) f(y) \dd y
    ,\quad
    \K_\mu f = \int k_\sigma(\cdot ,y) f(y) \dd \mu(y)
    .
\end{align}
Formally, as 
$k_\sigma(x, \cdot)\dd \mu \to\dd \delta_x $,
we have $\K_\mu \xi  \to \xi$ for any $\xi\in L^2(\mu)$.
Then, 
the \ours gradient flow equation
tends towards the PDE
\begin{align}
    \dot \mu = - \alpha\cdot \DIV (\mu\nabla \dFdmu) + \beta \mu\cdot\dFdmu
    .
\end{align}
Furthermore, we also have
$\|{\dFdmu}\|_\rkhs^2\to \|\dFdmu \|^2_{L^2(\mu )}$.
Hence, the conclusion follows.
\end{proof}

\begin{proof}
    [Proof of Theorem~\ref{thm:spherical-MMD-MMD-gfe}]
    We first recall that
    the first variation of the squared MMD energy is given by
    \begin{align}
        \frac{\delta }{\delta \mu  }
        \left(
            \frac12\mmd ^2(\mu, \nu)
        \right) 
        [\mu ]
        =
        \int 
        k(x, \cdot )
         (\mu-\nu )(\dd x)
        .
        \label{eq:mmd-first-var}
    \end{align}
    Plugging the first variation of the squared MMD energy into the gradient flow equation~\eqref{eq:mmd-gfe}, \eqref{eq:spherical-MMD-gfe},
    \eqref{eq:ikw-gfe-unreg}, and \eqref{eq:spherical-ours-gfe}, we obtain the desired flow equations in the theorem.
    The ODE solution is obtained by an elementary argument.
\end{proof}

\begin{proof}
    [Proof of Theorem~\ref{thm:mmd-energy-decay}]
    We take the time derivative of the energy and apply the chain rule formally and noting the gradient flow equation~\eqref{eq:ikw-gfe-unreg},
    \begin{multline*}
        \frac{\dd }{\dd t}
            F(\mu)
        =
        \langle
            \int 
            \dFdmu
             ,
             {\alpha}\cdot \operatorname{div}(\mu \nabla\dFdmu) -   {\beta} \cdot \K^{-1} \dFdmu
            \rangle_{L^2}
            \\
            =
 -\alpha\cdot \|\nabla \dFdmu\|^2_{L^2_{\mu}}
- \beta\cdot \| \dFdmu\|^2_{\rkhs}
\leq 
- \beta\cdot \| \dFdmu\|^2_{\rkhs}
            .
    \end{multline*}
Plugging in $F(\mu)=\frac12\mmd^2(\mu, \nu)$,
an elementary calculation shows that
\begin{align}
    \label{eq:mmd-energy-derivative}
    \frac{\dd }{\dd t}
    \left(
        \frac12\mmd^2(\mu, \nu)
    \right) 
    \leq 
    - \beta\cdot \| \int 
    k(x, \cdot )
     (\mu-\nu )(\dd x)\|^2_{\rkhs}
     =
        - 2 \beta\cdot \frac12\mmd^2(\mu, \nu)
    ,
\end{align}
which establishes the desired \Loj inequality
specialized to the squared MMD energy, which reads
\begin{align*}
   \alpha\cdot  \bigg\|\nabla \int 
    k(x, \cdot )
     (\mu-\nu )(\dd x)\bigg\|^2_{L^2_{\mu}}
    + 
    \beta\cdot  \bigg\|{\int 
    k(x, \cdot )
     (\mu-\nu )(\dd x)}\bigg\|^2_\rkhs
    \geq 
    c\cdot
    \frac12\mmd^2(\mu, \nu)
    \tag{\L{}oj}
    .
\end{align*}
By Gr\"onwall's lemma, exponential decay is established.

Furthermore, 
    plugging the first variation of the squared MMD energy into the gradient flow equation~\eqref{eq:spherical-MMD-gfe},
    the extra term in the spherical flow equation becomes
    \begin{align*}
        \frac{\int \K^{-1}\dFdmu}{\int \K^{-1}1}
        =
        \frac{ \int \K^{-1}\K(\mu-\pi)}{\int \K^{-1}1}
        =
        0
        .
    \end{align*}
    Hence, the coincidence is proved.
\end{proof}

\begin{proof}
    [Proof of Proposition~\ref{prop:kl-energy-decay}]
    By the \BE Theorem, we have the LSI~\eqref{eq:LSI} hold with 
    $c_{\textrm{LSI}}=2\lambda$.
    Taking the time derivative of the KL divergence energy along the S\ours gradient flow, we have
    \begin{multline*}
        \frac{\dd}{\dd t} \mathrm{D}_\mathrm{KL}(\mu_t|\pi)
        =
        \biggl\langle
        \nabla \log \frac{\dd\mu_t}{\dd\pi}
        ,
        \dot \mu_t
        \biggr\rangle_{L^2}
        \\
        =
        -\alpha\cdot\left\|\nabla \log \frac{\dd\mu_t}{\dd\pi}\right\|^2_{L^2_{\mu_t}}
        -\beta\cdot \left\|
        \log \frac{\dd\mu_t}{\dd\pi} - \frac{\int \K^{-1}\log \frac{\dd\mu_t}{\dd\pi}}{\int \K^{-1}1}
        \right\|^2_{\rkhs}
        \\
        \overset{\eqref{eq:LSI}}{\leq}
         -2\alpha \lambda\cdot \mathrm{D}_\mathrm{KL}(\mu_t|\pi)+0
        .
    \end{multline*}
    By Gr\"onwall's lemma, exponential convergence is established.
\end{proof}
Note that this result does not hold for the full IFT flows over non-negative measures as there exists no LSI globally on $\Mplus$.

\begin{remark}
    [Regularized inverse of the integral operator]
    Strictly speaking, 
    the integral operator $\K$ is compact and hence its inverse is unbounded.
    Using a viscosity-regularization techniques by \citet{efendiev2006rate},
    we can obtain the flow equation
    where the inverse is always well-defined, \ie 
    $\dot \mu =  \alpha\cdot \operatorname{div}(\mu \nabla\dFdmu) -   \beta \cdot (\K + \epsilon\cdot I)^{-1} \dFdmu$.
    This corresponds to an additive regularization of the kernel Gram matrix in practical computation.
\end{remark}

\paragraph*{A particle gradient descent algorithm for \ours gradient flows}
We use the notation $K_{XX}$ to denote the kernel Gram matrix $K_{XX} = [k(x^{\ell+1}_i, x^{\ell+1}_j)]_{i,j=1}^n$, $K_{X\bar X}$ for the cross kernel matrix $K_{X\bar X} = [k(x^{\ell+1}_i,  x^\ell_j)]_{i,j=1}^n$, etc.

\begin{algorithm}[hbt]
    \caption{A JKO-splitting for \ours particle gradient descent}
    {
        \begin{algorithmic}[1]
            \label{alg:jko-split}
        \REQUIRE{ }
        \FOR{$ \ell =  1$ to $T - 1$}
            \STATE{Compute the first variation of the energy $F$ at $\mu^\ell$: $        g^\ell =
            \dFdmuk
            $. 
Then,
                \begin{align*}
                    x_i^{\ell+1}
                    &\gets
                    x^\ell_i - \tau^\ell \cdot 
                    \nabla 
                    g^\ell
                    (x^\ell_i), \quad i=1,...,n
                    && \textrm{(Wasserstein step)}
                    \\
                    \alpha^{\ell+1}
                    &
                    \gets\argmin_{\alpha\in \Delta^n}
                    F( \sum_{i=1}^{n}\alpha_i\delta_{x_i^{\ell+1}})
                      + \frac1{2\eta^\ell}
                    \begin{bmatrix}
                    \alpha \\
                    \alpha^\ell
                    \end{bmatrix}^\top
                      \begin{pmatrix}K_{XX}& - K_{X \bar X} \\
                        - K_{X\bar X}& K_{\bar X\bar X}\end{pmatrix}
                        \begin{bmatrix}
                            \alpha \\
                            \alpha^\ell
                            \end{bmatrix}
                    && \textrm{(MMD step)}
                \end{align*}  
                }
        \ENDFOR
        \STATE Output the particle measure $ \widehat \mu^T = \sum_{i=1}^{n}\alpha^T_i\delta_{x_i^T}$.
    \end{algorithmic}
    }
    \end{algorithm}

\paragraph*{Implementing the MMD step}
The MMD step in the JKO-splitting scheme is a convex quadratic program with a simplex constraint, which can be formulated as
\begin{align}\label{mmd_st}
    \inf_{\beta\in \Delta^n} 
    \frac12\mmd^2( \sum_{i=1}^{n}\beta_i\delta_{x_i^{\ell+1}} , \pi)
        +
        \frac1{2\eta}
    \mmd^2(  \sum_{i=1}^{n}\beta_i\delta_{x^{\ell+1}_i},  \sum_{i=1}^{n}\alpha^\ell_i\delta_{x^{\ell+1}_i}).
\end{align}
We further expand the optimization objective
(multiplied by a
factor of $2\tau$ for convenience)
\begin{multline}
    \tau\cdot \left\|
     \sum_{i=1}^{n}\beta_i\phi({x_i^{k+1}})
     -
     \frac1m\sum_{j=1}^{m}\phi({y_j})
     \right\|^2
     +
    \left\|
        \sum_{i=1}^{n}\beta_i\phi({x_i^{k+1}})
        -
        \sum_{i=1}^{n}\alpha^k_i\phi({x^k_i})
        \right\|^2
        \\
= 
\tau\cdot 
\left(
    \beta^\top K_{XX} \beta
    -\frac2m\beta^\top K_{XY} \mathbf{1}
    +
    \frac1{m^2}\mathbf{1}^\top K_{YY} \mathbf{1}
    \right)
+
\left(
    \beta^\top K_{XX} \beta
    -2\beta^\top K_{X\bar{X}} \alpha
    +\alpha^{\top} K_{\bar X\bar X} \alpha
    \right)
    \\
=
\left(1+\tau\right)
\beta^\top K_{XX} \beta
-\frac{2\tau}m\beta^\top K_{XY} \mathbf{1}
-2\beta^\top K_{X\bar{X}} \alpha
+
\frac\tau{m^2}\mathbf{1}^\top K_{YY} \mathbf{1}
+\alpha^{\top} K_{\bar X\bar X} \alpha
    .
\end{multline}
Therefore, the MMD step in Algorithm~\ref{alg:jko-split} can be implemented as
a convex quadratic program with a simplex constraint.

\paragraph*{A particle gradient descent algorithm for the WFR flow of the MMD energy}
We provide the implementation details of the WFR flow of the MMD energy.
The goal is to simulate the PDE \eqref{eq:wfr-MMD-gfe}.
To the best of our knowledge, there has been no prior implementation of this flow. Nor is there a convergence guarantee.
Similar to the JKO-splitting scheme of the \ours flow in \eqref{eq:jko-split}, \eqref{eq:wfr-MMD-gfe} can be discretized using the two-step scheme
\begin{align*}
    \mu^{\ell+\frac12}
    &\gets\arg\min_{\mu\in\cal P} F(\mu ) + \frac1{2\tau}W_2^2(\mu, \mu^\ell)
    \quad\textrm{(Wasserstein step)} \\
    \mu^{\ell+1}
    &\gets\argmin_{\mu\in\cal P} F(\mu ) + \frac1{\eta}{\mathrm{KL}}(\mu, \mu^{\ell+\frac12})
    \quad\textrm{(KL step)}
\end{align*}
where the energy function $F$ is the squared MMD energy, $F(\mu)=\frac12\mmd^2(\mu, \pi)$.
Use the
explicit Euler scheme,
the KL step amounts to the entropic mirror descent.
In the optimization literature, this step can be implemented as multiplicative update of the weights (or density), i.e.,
suppose $x_i^{\ell+1}$ is the new particle location after the Wasserstein step, then we update the weights vector $\alpha$ via
\begin{align*}
    \alpha_i ^{\ell+1}
    \gets \alpha_i  ^{\ell} \cdot \exp
    \left(
        -\eta \cdot 
        \frac{\delta F}{\delta \mu}[\mu^\ell] (x^{\ell+1}_i)
    \right).
\end{align*}

\begin{ack}
    We thank Gabriel Peyré
    for the helpful comments regarding the practical algorithms for the JKO-splitting scheme.
    This project has received funding from the Deutsche Forschungsgemeinschaft (DFG, German Research Foundation) under Germany's Excellence Strategy – The Berlin Mathematics Research Center MATH+ (EXC-2046/1, project ID: 390685689)
    and from the priority programme "Theoretical Foundations of Deep Learning" (SPP 2298, project number: 543963649). 
    During part of the project,
    the research of E. Gladin was prepared within the framework of the HSE University Basic Research Program.
    \end{ack}
\bibliographystyle{abbrvnat}
\bibliography{ref}

\begin{thebibliography}{56}
\providecommand{\natexlab}[1]{#1}
\providecommand{\url}[1]{\texttt{#1}}
\expandafter\ifx\csname urlstyle\endcsname\relax
  \providecommand{\doi}[1]{doi: #1}\else
  \providecommand{\doi}{doi: \begingroup \urlstyle{rm}\Url}\fi

\bibitem[Ambrosio et~al.(2005)Ambrosio, Gigli, and
  Savare]{ambrosio2008gradient}
L.~Ambrosio, N.~Gigli, and G.~Savare.
\newblock \emph{Gradient Flows: In Metric Spaces and in the Space of
  Probability Measures}.
\newblock Springer Science \& Business Media, 2005.

\bibitem[Arbel et~al.(2019)Arbel, Korba, SALIM, and
  Gretton]{arbel_maximum_2019}
M.~Arbel, A.~Korba, A.~SALIM, and A.~Gretton.
\newblock Maximum mean discrepancy gradient flow.
\newblock In H.~Wallach, H.~Larochelle, A.~Beygelzimer, F.~d\textquotesingle
  Alch\'{e}-Buc, E.~Fox, and R.~Garnett, editors, \emph{Advances in Neural
  Information Processing Systems}, volume~32. Curran Associates, Inc., 2019.
\newblock URL
  \url{https://proceedings.neurips.cc/paper_files/paper/2019/file/944a5ae3483ed5c1e10bbccb7942a279-Paper.pdf}.

\bibitem[Arjovsky et~al.(2017)Arjovsky, Chintala, and
  Bottou]{arjovsky2017wasserstein}
M.~Arjovsky, S.~Chintala, and L.~Bottou.
\newblock Wasserstein generative adversarial networks.
\newblock In \emph{International conference on machine learning}, pages
  214--223. PMLR, 2017.

\bibitem[Bakry and {\'E}mery(1985)]{bakryDiffusionsHypercontractives1985}
D.~Bakry and M.~{\'E}mery.
\newblock Diffusions hypercontractives.
\newblock In J.~Az{\'e}ma and M.~Yor, editors, \emph{S{\'e}minaire de
  {{Probabilit{\'e}s XIX}} 1983/84}, volume 1123, pages 177--206. {Springer
  Berlin Heidelberg}, {Berlin, Heidelberg}, 1985.
\newblock ISBN 978-3-540-15230-9 978-3-540-39397-9.
\newblock \doi{10.1007/BFb0075847}.

\bibitem[Bigot et~al.(2017)Bigot, Gouet, Klein, and L\'opez]{bigot2017geodesic}
J.~Bigot, R.~Gouet, T.~Klein, and A.~L\'opez.
\newblock Geodesic {PCA} in the {Wasserstein} space by convex {PCA}.
\newblock \emph{Ann. Inst. H. Poincar\'e Probab. Statist.}, 53\penalty0
  (1):\penalty0 1--26, 02 2017.
\newblock \doi{10.1214/15-AIHP706}.
\newblock URL \url{https://doi.org/10.1214/15-AIHP706}.

\bibitem[Carrillo et~al.(2019)Carrillo, Craig, and
  Patacchini]{carrillo2019blob}
J.~A. Carrillo, K.~Craig, and F.~S. Patacchini.
\newblock A blob method for diffusion.
\newblock \emph{Calculus of Variations and Partial Differential Equations},
  58:\penalty0 1--53, 2019.

\bibitem[Chewi et~al.(2020)Chewi, Le~Gouic, Lu, Maunu, and
  Rigollet]{chewiSVGDKernelizedWasserstein2020}
S.~Chewi, T.~Le~Gouic, C.~Lu, T.~Maunu, and P.~Rigollet.
\newblock {SVGD} as a kernelized {Wasserstein} gradient flow of the chi-squared
  divergence.
\newblock \emph{Advances in Neural Information Processing Systems},
  33:\penalty0 2098--2109, 2020.

\bibitem[Chizat(2022)]{chizat2022sparse}
L.~Chizat.
\newblock Sparse optimization on measures with over-parameterized gradient
  descent.
\newblock \emph{Mathematical Programming}, 194\penalty0 (1-2):\penalty0
  487--532, 2022.

\bibitem[Chizat et~al.(2018)Chizat, Peyr{\'e}, Schmitzer, and
  Vialard]{chizatInterpolatingDistanceOptimal2018}
L.~Chizat, G.~Peyr{\'e}, B.~Schmitzer, and F.-X. Vialard.
\newblock An interpolating distance between optimal transport and
  {{Fisher}}--{{Rao}} metrics.
\newblock \emph{Foundations of Computational Mathematics}, 18\penalty0
  (1):\penalty0 1--44, Feb. 2018.
\newblock ISSN 1615-3375, 1615-3383.
\newblock \doi{10.1007/s10208-016-9331-y}.

\bibitem[Chizat et~al.(2019)Chizat, Peyré, Schmitzer, and
  Vialard]{chizat_unbalanced_2019}
L.~Chizat, G.~Peyré, B.~Schmitzer, and F.-X. Vialard.
\newblock Unbalanced optimal transport: Dynamic and {Kantorovich} formulation.
\newblock \emph{arXiv:1508.05216 [math]}, Feb. 2019.
\newblock URL \url{http://arxiv.org/abs/1508.05216}.
\newblock arXiv: 1508.05216.

\bibitem[Cohen et~al.(2021)Cohen, Arbel, and
  Deisenroth]{cohenEstimatingBarycentersMeasures2021}
S.~Cohen, M.~Arbel, and M.~P. Deisenroth.
\newblock Estimating barycenters of measures in high dimensions.
\newblock \emph{arXiv:2007.07105 [cs, stat]}, Feb. 2021.

\bibitem[Craig et~al.(2023)Craig, Elamvazhuthi, Haberland, and
  Turanova]{craig_blob_2023}
K.~Craig, K.~Elamvazhuthi, M.~Haberland, and O.~Turanova.
\newblock A blob method for inhomogeneous diffusion with applications to
  multi-agent control and sampling.
\newblock \emph{Mathematics of Computation}, 92\penalty0 (344):\penalty0
  2575--2654, Nov. 2023.
\newblock ISSN 0025-5718, 1088-6842.
\newblock \doi{10.1090/mcom/3841}.

\bibitem[Cuturi(2013)]{cuturi2013sinkhorn}
M.~Cuturi.
\newblock Sinkhorn distances: Lightspeed computation of optimal transport.
\newblock \emph{Advances in neural information processing systems}, 26, 2013.

\bibitem[Dai et~al.(2016)Dai, He, Dai, and
  Song]{daiProvableBayesianInference2016}
B.~Dai, N.~He, H.~Dai, and L.~Song.
\newblock Provable {{Bayesian}} inference via particle mirror descent.
\newblock In \emph{Proceedings of the 19th {{International Conference}} on
  {{Artificial Intelligence}} and {{Statistics}}}, pages 985--994. PMLR, May
  2016.

\bibitem[De~Giorgi(1993)]{de1993new}
E.~De~Giorgi.
\newblock New problems on minimizing movements.
\newblock \emph{Ennio de Giorgi: Selected Papers}, pages 699--713, 1993.

\bibitem[Duncan et~al.(2019)Duncan, N{\"u}sken, and
  Szpruch]{duncan2019geometry}
A.~Duncan, N.~N{\"u}sken, and L.~Szpruch.
\newblock On the geometry of {Stein} variational gradient descent.
\newblock \emph{arXiv preprint arXiv:1912.00894}, 2019.

\bibitem[Efendiev and Mielke(2006)]{efendiev2006rate}
M.~A. Efendiev and A.~Mielke.
\newblock On the rate-independent limit of systems with dry friction and small
  viscosity.
\newblock \emph{Journal of Convex Analysis}, 13\penalty0 (1):\penalty0 151,
  2006.

\bibitem[Feydy et~al.(2019)Feydy, S{\'e}journ{\'e}, Vialard, Amari, Trouve, and
  Peyr{\'e}]{feydyInterpolatingOptimalTransport2019}
J.~Feydy, T.~S{\'e}journ{\'e}, F.-X. Vialard, S.-i. Amari, A.~Trouve, and
  G.~Peyr{\'e}.
\newblock Interpolating between optimal transport and {{MMD}} using
  {{Sinkhorn}} divergences.
\newblock In \emph{Proceedings of the {{Twenty-Second International
  Conference}} on {{Artificial Intelligence}} and {{Statistics}}}, pages
  2681--2690. PMLR, Apr. 2019.

\bibitem[Gallou{\"e}t and Monsaingeon(2017)]{gallouet2017jko}
T.~O. Gallou{\"e}t and L.~Monsaingeon.
\newblock A {JKO} splitting scheme for {Kantorovich--Fisher--Rao} gradient
  flows.
\newblock \emph{SIAM Journal on Mathematical Analysis}, 49\penalty0
  (2):\penalty0 1100--1130, 2017.

\bibitem[Glaser et~al.(2021)Glaser, Arbel, and
  Gretton]{glaserKALEFlowRelaxed2021}
P.~Glaser, M.~Arbel, and A.~Gretton.
\newblock {{KALE}} flow: A relaxed {{KL}} gradient flow for probabilities with
  disjoint support.
\newblock In \emph{Neural {{Information Processing Systems}}}, June 2021.

\bibitem[Gretton et~al.(2012)Gretton, Borgwardt, Rasch, Sch{\"o}lkopf, and
  Smola]{gretton2012kernel}
A.~Gretton, K.~M. Borgwardt, M.~J. Rasch, B.~Sch{\"o}lkopf, and A.~Smola.
\newblock A kernel two-sample test.
\newblock \emph{The Journal of Machine Learning Research}, 13\penalty0
  (1):\penalty0 723--773, 2012.

\bibitem[Hertrich et~al.(2023)Hertrich, Wald, Altekr{\"u}ger, and
  Hagemann]{hertrich2023generative}
J.~Hertrich, C.~Wald, F.~Altekr{\"u}ger, and P.~Hagemann.
\newblock Generative sliced {MMD} flows with {Riesz} kernels.
\newblock \emph{arXiv preprint arXiv:2305.11463}, 2023.

\bibitem[Ho et~al.(2017)Ho, Nguyen, Yurochkin, Bui, Huynh, and
  Phung]{ho17multilevel}
N.~Ho, X.~Nguyen, M.~Yurochkin, H.~H. Bui, V.~Huynh, and D.~Phung.
\newblock Multilevel clustering via {W}asserstein means.
\newblock In D.~Precup and Y.~W. Teh, editors, \emph{Proceedings of the 34th
  International Conference on Machine Learning}, volume~70 of \emph{Proceedings
  of Machine Learning Research}, pages 1501--1509, International Convention
  Centre, Sydney, Australia, 06--11 Aug 2017. PMLR.
\newblock URL \url{http://proceedings.mlr.press/v70/ho17a.html}.

\bibitem[Jordan et~al.(1998)Jordan, Kinderlehrer, and
  Otto]{jordan_variational_1998}
R.~Jordan, D.~Kinderlehrer, and F.~Otto.
\newblock The variational formulation of the {Fokker}–{Planck} equation.
\newblock \emph{SIAM journal on mathematical analysis}, 29\penalty0
  (1):\penalty0 1--17, 1998.
\newblock Publisher: SIAM.

\bibitem[Kondratyev et~al.(2016)Kondratyev, Monsaingeon, and
  Vorotnikov]{kondratyevNewOptimalTransport2016}
S.~Kondratyev, L.~Monsaingeon, and D.~Vorotnikov.
\newblock A new optimal transport distance on the space of finite {Radon}
  measures.
\newblock \emph{Advances in Differential Equations}, 21\penalty0
  (11/12):\penalty0 1117 -- 1164, 2016.
\newblock \doi{10.57262/ade/1476369298}.
\newblock URL \url{https://doi.org/10.57262/ade/1476369298}.

\bibitem[Korba et~al.(2021)Korba, {Aubin-Frankowski}, Majewski, and
  Ablin]{korbaKernelSteinDiscrepancy2021}
A.~Korba, P.-C. {Aubin-Frankowski}, S.~Majewski, and P.~Ablin.
\newblock Kernel {{Stein}} discrepancy descent.
\newblock In \emph{Proceedings of the 38th {{International Conference}} on
  {{Machine Learning}}}, pages 5719--5730. {PMLR}, July 2021.

\bibitem[Kusner et~al.(2015)Kusner, Sun, Kolkin, and
  Weinberger]{kusner2015from}
M.~J. Kusner, Y.~Sun, N.~I. Kolkin, and K.~Q. Weinberger.
\newblock From word embeddings to document distances.
\newblock In \emph{Proceedings of the 32nd International Conference on
  International Conference on Machine Learning - Volume 37}, ICML'15, pages
  957--966. JMLR.org, 2015.
\newblock URL \url{http://dl.acm.org/citation.cfm?id=3045118.3045221}.

\bibitem[Lacoste-Julien et~al.(2015)Lacoste-Julien, Lindsten, and
  Bach]{lacoste-julienSequentialKernelHerding}
S.~Lacoste-Julien, F.~Lindsten, and F.~Bach.
\newblock Sequential kernel herding: {Frank-Wolfe} optimization for particle
  filtering.
\newblock In G.~Lebanon and S.~V.~N. Vishwanathan, editors, \emph{Proceedings
  of the Eighteenth International Conference on Artificial Intelligence and
  Statistics}, volume~38 of \emph{Proceedings of Machine Learning Research},
  pages 544--552, San Diego, California, USA, 09--12 May 2015. PMLR.
\newblock URL \url{https://proceedings.mlr.press/v38/lacoste-julien15.html}.

\bibitem[Laschos and Mielke(2019)]{LasMie19GPCA}
V.~Laschos and A.~Mielke.
\newblock Geometric properties of cones with applications on the
  {H}ellinger--{K}antorovich space, and a new distance on the space of
  probability measures.
\newblock \emph{J. Funct. Analysis}, 276\penalty0 (11):\penalty0 3529--3576,
  2019.
\newblock \doi{10.1016/j.jfa.2018.12.013}.

\bibitem[Liero et~al.(2018)Liero, Mielke, and Savaré]{liero_optimal_2018}
M.~Liero, A.~Mielke, and G.~Savaré.
\newblock Optimal entropy-transport problems and a new
  {Hellinger}–{Kantorovich} distance between positive measures.
\newblock \emph{Inventiones mathematicae}, 211\penalty0 (3):\penalty0
  969--1117, Mar. 2018.
\newblock ISSN 0020-9910, 1432-1297.
\newblock \doi{10.1007/s00222-017-0759-8}.
\newblock URL \url{http://link.springer.com/10.1007/s00222-017-0759-8}.

\bibitem[Liero et~al.(2023)Liero, Mielke, and Savar\'e]{LiMiSa23FPGG}
M.~Liero, A.~Mielke, and G.~Savar\'e.
\newblock Fine properties of geodesics and geodesic $\lambda$-convexity for the
  {Hellinger--Kantorovich} distance.
\newblock \emph{Arch. Rat. Mech. Analysis}, 247\penalty0 (112):\penalty0 1--73,
  2023.
\newblock \doi{10.1007/s00205-023-01941-1}.

\bibitem[Liu and Wang(2019)]{liu_stein_2019}
Q.~Liu and D.~Wang.
\newblock Stein variational gradient descent: A general purpose {Bayesian}
  inference algorithm.
\newblock \emph{arXiv:1608.04471 [cs, stat]}, Sept. 2019.
\newblock URL \url{http://arxiv.org/abs/1608.04471}.
\newblock arXiv: 1608.04471.

\bibitem[Lu et~al.(2019)Lu, Lu, and Nolen]{luAcceleratingLangevinSampling2019}
Y.~Lu, J.~Lu, and J.~Nolen.
\newblock Accelerating {{Langevin}} sampling with birth-death.
\newblock \emph{ArXiv}, May 2019.

\bibitem[Lu et~al.(2023)Lu, Slep{\v{c}}ev, and Wang]{lu2023birth}
Y.~Lu, D.~Slep{\v{c}}ev, and L.~Wang.
\newblock Birth--death dynamics for sampling: global convergence,
  approximations and their asymptotics.
\newblock \emph{Nonlinearity}, 36\penalty0 (11):\penalty0 5731, 2023.

\bibitem[Manupriya et~al.(2024)Manupriya, Jagarlapudi, and
  Jawanpuria]{manupriyaMMDRegularizedUnbalancedOptimal2024}
P.~Manupriya, S.~N. Jagarlapudi, and P.~Jawanpuria.
\newblock {MMD}-regularized unbalanced optimal transport.
\newblock \emph{Transactions on Machine Learning Research}, 2024.
\newblock ISSN 2835-8856.
\newblock URL \url{https://openreview.net/forum?id=eN9CjU3h1b}.

\bibitem[Mielke et~al.(2023)Mielke, Rossi, and
  Stephan]{mielkeTimesplittingMethodsGradient2023}
A.~Mielke, R.~Rossi, and A.~Stephan.
\newblock On time-splitting methods for gradient flows with two dissipation
  mechanisms.
\newblock \emph{arXiv preprint arXiv:2307.16137}, 2023.

\bibitem[Mohajerin~Esfahani and
  Kuhn(2018)]{mohajerinesfahaniDatadrivenDistributionallyRobust2018}
P.~Mohajerin~Esfahani and D.~Kuhn.
\newblock Data-driven distributionally robust optimization using the
  {{Wasserstein}} metric: Performance guarantees and tractable reformulations.
\newblock \emph{Mathematical Programming}, 171\penalty0 (1-2):\penalty0
  115--166, Sept. 2018.
\newblock ISSN 0025-5610, 1436-4646.
\newblock \doi{10.1007/s10107-017-1172-1}.

\bibitem[Mroueh and Rigotti(2020)]{mroueh_unbalanced_2020}
Y.~Mroueh and M.~Rigotti.
\newblock Unbalanced {Sobolev} descent, Sept. 2020.
\newblock URL \url{http://arxiv.org/abs/2009.14148}.
\newblock arXiv:2009.14148 [cs, stat].

\bibitem[Muandet et~al.(2017)Muandet, Fukumizu, Sriperumbudur, and
  Sch{\"o}lkopf]{muandetKernelMeanEmbedding2017}
K.~Muandet, K.~Fukumizu, B.~Sriperumbudur, and B.~Sch{\"o}lkopf.
\newblock Kernel mean embedding of distributions: A review and beyond.
\newblock \emph{Foundations and Trends{\textregistered} in Machine Learning},
  10\penalty0 (1-2):\penalty0 1--141, 2017.
\newblock ISSN 1935-8237, 1935-8245.
\newblock \doi{10.1561/2200000060}.

\bibitem[Nemirovskij and Yudin(1983)]{nemirovskijProblemComplexityMethod1983}
A.~S. Nemirovskij and D.~B. Yudin.
\newblock \emph{Problem Complexity and Method Efficiency in Optimization}.
\newblock John Wiley, New York, 1983.

\bibitem[Neumayer et~al.(2024)Neumayer, Stein, and
  Steidl]{neumayer2024wasserstein}
S.~Neumayer, V.~Stein, and G.~Steidl.
\newblock Wasserstein gradient flows for {Moreau} envelopes of f-divergences in
  reproducing kernel {Hilbert} spaces.
\newblock \emph{arXiv preprint arXiv:2402.04613}, 2024.

\bibitem[Otto(1996)]{otto1996double}
F.~Otto.
\newblock Double degenerate diffusion equations as steepest descent.
\newblock Bonn University, 1996.
\newblock Preprint.

\bibitem[Rubner et~al.(2000)Rubner, Tomasi, and Guibas]{rubner2000earth}
Y.~Rubner, C.~Tomasi, and L.~J. Guibas.
\newblock The earth mover's distance as a metric for image retrieval.
\newblock \emph{International journal of computer vision}, 40\penalty0
  (2):\penalty0 99--121, 2000.

\bibitem[Sandler and Lindenbaum(2011)]{sandler2011nonnegative}
R.~Sandler and M.~Lindenbaum.
\newblock Nonnegative matrix factorization with earth mover's distance metric
  for image analysis.
\newblock \emph{IEEE Transactions on Pattern Analysis and Machine
  Intelligence}, 33\penalty0 (8):\penalty0 1590--1602, 2011.
\newblock ISSN 0162-8828.
\newblock \doi{10.1109/TPAMI.2011.18}.

\bibitem[Santambrogio(2015)]{santambrogio_optimal_2015}
F.~Santambrogio.
\newblock \emph{Optimal Transport for Applied Mathematicians: Calculus of
  Variations, {PDEs}, and Modeling}.
\newblock Springer International Publishing, 2015.
\newblock ISBN 9783319208282.
\newblock \doi{10.1007/978-3-319-20828-2}.
\newblock URL \url{http://dx.doi.org/10.1007/978-3-319-20828-2}.

\bibitem[Sinha et~al.(2020)Sinha, Namkoong, Volpi, and
  Duchi]{sinhaCertifyingDistributionalRobustness2020}
A.~Sinha, H.~Namkoong, R.~Volpi, and J.~Duchi.
\newblock Certifying some distributional robustness with principled adversarial
  training.
\newblock \emph{arXiv:1710.10571 [cs, stat]}, May 2020.

\bibitem[Smola et~al.(2007)Smola, Gretton, Song, and
  Sch{\"o}lkopf]{smolaHilbertSpaceEmbedding2007}
A.~Smola, A.~Gretton, L.~Song, and B.~Sch{\"o}lkopf.
\newblock A {Hilbert} space embedding for distributions.
\newblock In M.~Hutter, R.~A. Servedio, and E.~Takimoto, editors,
  \emph{Algorithmic Learning Theory}, pages 13--31, {Berlin, Heidelberg}, 2007.
  {Springer}.
\newblock ISBN 978-3-540-75225-7.
\newblock \doi{10.1007/978-3-540-75225-7_5}.

\bibitem[Solomon et~al.(2014)Solomon, Rustamov, Guibas, and
  Butscher]{solomon2014wasserstein}
J.~Solomon, R.~M. Rustamov, L.~Guibas, and A.~Butscher.
\newblock Wasserstein propagation for semi-supervised learning.
\newblock In \emph{Proceedings of the 31st International Conference on
  International Conference on Machine Learning - Volume 32}, ICML'14, pages
  I--306--I--314. JMLR.org, 2014.
\newblock URL \url{http://dl.acm.org/citation.cfm?id=3044805.3044841}.

\bibitem[Steinwart and Christmann(2008)]{steinwart2008support}
I.~Steinwart and A.~Christmann.
\newblock \emph{Support vector machines}.
\newblock Springer Science \& Business Media, 2008.

\bibitem[Tolstikhin et~al.(2017)Tolstikhin, Sriperumbudur, and
  Muandet]{tolstikhinMinimaxEstimationKernel2017}
I.~Tolstikhin, B.~Sriperumbudur, and K.~Muandet.
\newblock Minimax estimation of kernel mean embeddings.
\newblock \emph{arXiv:1602.04361 [math, stat]}, July 2017.

\bibitem[Tolstikhin et~al.(2016)Tolstikhin, Sriperumbudur, and
  Sch\"{o}lkopf]{tolstikhinMinimaxEstimationMaximum}
I.~O. Tolstikhin, B.~K. Sriperumbudur, and B.~Sch\"{o}lkopf.
\newblock Minimax estimation of maximum mean discrepancy with radial kernels.
\newblock In D.~Lee, M.~Sugiyama, U.~Luxburg, I.~Guyon, and R.~Garnett,
  editors, \emph{Advances in Neural Information Processing Systems}, volume~29.
  Curran Associates, Inc., 2016.
\newblock URL
  \url{https://proceedings.neurips.cc/paper_files/paper/2016/file/5055cbf43fac3f7e2336b27310f0b9ef-Paper.pdf}.

\bibitem[Tsybakov(2009)]{tsybakov_introduction_2009}
A.~B. Tsybakov.
\newblock \emph{Introduction to {Nonparametric} {Estimation}}.
\newblock Springer {Series} in {Statistics}. Springer, New York, NY, 2009.
\newblock ISBN 978-0-387-79051-0 978-0-387-79052-7.
\newblock \doi{10.1007/b13794}.
\newblock URL \url{https://link.springer.com/10.1007/b13794}.

\bibitem[Villani(2008)]{villani2008optimal}
C.~Villani.
\newblock \emph{Optimal transport: old and new}, volume 338.
\newblock Springer Science \& Business Media, 2008.

\bibitem[Werman et~al.(1985)Werman, Peleg, and Rosenfeld]{werman1985distance}
M.~Werman, S.~Peleg, and A.~Rosenfeld.
\newblock A distance metric for multidimensional histograms.
\newblock \emph{Computer Vision, Graphics, and Image Processing}, 32\penalty0
  (3):\penalty0 328 -- 336, 1985.
\newblock ISSN 0734-189X.
\newblock \doi{https://doi.org/10.1016/0734-189X(85)90055-6}.
\newblock URL
  \url{http://www.sciencedirect.com/science/article/pii/0734189X85900556}.

\bibitem[Yan et~al.(2024)Yan, Wang, and
  Rigollet]{yanLearningGaussianMixtures2023}
Y.~Yan, K.~Wang, and P.~Rigollet.
\newblock Learning {Gaussian} mixtures using the {Wasserstein-Fisher-Rao}
  gradient flow.
\newblock \emph{The Annals of Statistics}, 52\penalty0 (4):\penalty0
  1774--1795, 2024.

\bibitem[Zhu and Mielke(2024)]{zhu2024approximation}
J.-J. Zhu and A.~Mielke.
\newblock Kernel approximation of {Fisher-Rao} gradient flows, 2024.
\newblock URL \url{https://arxiv.org/abs/2410.20622}.

\end{thebibliography}

\end{document}